\documentclass{article}
\usepackage[final]{neurips_2020}

\setcitestyle{authoryear}

\usepackage[utf8]{inputenc} 
\usepackage[T1]{fontenc}    
\usepackage[colorlinks=true, allcolors=blue]{hyperref}       
\usepackage{url}            
\usepackage{booktabs}       
\usepackage{amsfonts}       
\usepackage{nicefrac}       
\usepackage{microtype}      
\usepackage[dvipsnames]{xcolor}
\usepackage{amsmath,amsthm,amssymb}
\usepackage{graphicx}
\usepackage{cleveref}
\usepackage{subcaption}
\usepackage{enumitem}
\usepackage{algorithm}
\usepackage{algorithmic}
\usepackage{bbm}
\usepackage{thmtools}

\usepackage{comment}

\theoremstyle{plain}
\newtheorem{theorem}{Theorem}[section]
\newtheorem{lemma}[theorem]{Lemma}
\newtheorem{remark}[theorem]{Remark}

\theoremstyle{definition}

\newtheorem{observation}[theorem]{Observation}
\newtheorem{conjecture}[theorem]{Conjecture}

\usepackage{amsmath,amsfonts,bm}

\def\Figref#1{Figure~\ref{#1}}

\def\1{\bm{1}}

\def\eps{{\epsilon}}

\def\vzero{{\bm{0}}}

\def\vmu{{\bm{\mu}}}
\def\vtheta{{\bm{\theta}}}

\def\vxi{{\bm{\xi}}}

\makeatletter
\newcommand{\ve}{\@ifnextchar\bgroup{\velong}{{\bm{e}}}}
\newcommand{\velong}[1]{{\bm{#1}}}
\makeatother

\def\vg{{\bm{g}}}

\def\vm{{\bm{m}}}

\def\vv{{\bm{v}}}
\def\vw{{\bm{w}}}
\def\vx{{\bm{x}}}

\def\mB{{\bm{B}}}

\def\mG{{\bm{G}}}

\def\mI{{\bm{I}}}

\def\mU{{\bm{U}}}

\def\mW{{\bm{W}}}
\def\mX{{\bm{X}}}

\def\mSigma{{\bm{\Sigma}}}

\DeclareMathAlphabet{\mathsfit}{\encodingdefault}{\sfdefault}{m}{sl}
\SetMathAlphabet{\mathsfit}{bold}{\encodingdefault}{\sfdefault}{bx}{n}

\newcommand{\E}{\mathbb{E}}
\newcommand{\R}{\mathbb{R}}

\newcommand{\normltwo}{L^2}

\DeclareMathOperator{\Tr}{Tr}

\newcommand{\adamcolor}{Thistle}
\newcommand{\ouradamcolor}{SpringGreen}
\newcommand{\adam}[1]{\colorbox{\adamcolor}{$\displaystyle #1$}}
\newcommand{\adamtext}[1]{\colorbox{\adamcolor}{#1}}
\newcommand{\ouradam}[1]{\colorbox{\ouradamcolor}{$\displaystyle #1$}}
\newcommand{\ouradamtext}[1]{\colorbox{\ouradamcolor}{#1}}

\newcommand{\BN}{\mathrm{BN}}

\newcommand{\dotp}[2]{\left<#1, #2\right>}

\newcommand{\Loss}{\mathcal{L}}

\newcommand{\Batch}{\mathcal{B}}

\newcommand{\norm}[1]{\left\| #1 \right\|}
\newcommand{\normtwo}[1]{\left\| #1 \right\|_2}
\newcommand{\normtwosm}[1]{\| #1 \|_2}

\newcommand{\dd}{\textup{\textrm{d}}}

\newcommand{\myparagraph}[1]{\paragraph{#1}}

\def\vW{{\bm{W}}}
\def\vB{{\bm{B}}}

\newcommand{\hW}{\overline{\vW}}
\newcommand{\hvw}{\bar{\vw}}

\newcommand{\green}[1]{{\color{Green} #1}}
\newcommand{\red}[1]{{\color{red} #1}}
\newcommand{\blue}[1]{{\color{NavyBlue} #1}}
\newcommand{\orange}[1]{{\color{Orange} #1}}
\newcommand{\purple}[1]{{\color{Purple} #1}}

\def\eqref#1{(\ref{#1})}

\newcommand{\dTV}{d_{\mathrm{TV}}}

\title{Reconciling Modern Deep Learning with Traditional Optimization Analyses: The Intrinsic Learning Rate}

\author{%
Zhiyuan Li\thanks{These authors contribute equally.} \\
  Princeton University\\
  \texttt{zhiyuanli@cs.princeton.edu} \\
 \And
Kaifeng Lyu \footnotemark[1]\\
  Tsinghua University\\
  \texttt{vfleaking@gmail.com} \\
\And
Sanjeev Arora\\
  Princeton University \& IAS\\
  \texttt{arora@cs.princeton.edu}
}

\begin{document}

\maketitle

\begin{abstract}

Recent works (e.g., \citep{li2020exp}) suggest that the use of popular normalization schemes  (including Batch Normalization) in today's deep learning can move it far from a traditional optimization viewpoint, e.g., use of exponentially increasing learning rates. The current paper  highlights other ways in which behavior of normalized nets departs from traditional viewpoints, and then initiates a formal framework for studying their mathematics via suitable adaptation of  the conventional framework  namely, modeling SGD-induced training trajectory via a suitable stochastic differential equation (SDE) with a noise term that captures gradient noise. This yields: (a) A new \textquotedblleft intrinsic learning rate\textquotedblright\ parameter that is the product of the normal learning rate $\eta$ and weight decay factor $\lambda$. Analysis of the SDE shows how the effective speed of learning varies and equilibrates over time under the control of intrinsic LR. (b) A challenge---via theory and experiments---to popular belief that good generalization requires large learning rates at the start of training. 
 (c) New experiments, backed by mathematical intuition, suggesting the number of steps to equilibrium (in function space) scales as the inverse of the intrinsic learning rate,  as opposed to the  exponential time convergence bound implied by SDE analysis. We name it the \emph{Fast Equilibrium Conjecture} and suggest it holds the key to why  Batch Normalization is effective. 
\end{abstract}

\section{Introduction}
The training of modern neural networks involves Stochastic Gradient Descent (SGD) with an appropriate learning rate schedule. The formula of SGD with weight decay can be written as:
\[
	\vw_{t+1} \gets (1 - \eta_t\lambda) \vw_t - \eta_t \nabla \Loss_t(\vw_t),
\]
where $\lambda$ is the weight decay factor (or $\normltwo$-regularization coefficient), $\eta_t$ and $\nabla \Loss_t(\vw_t)$ are the learning rate and batch gradient at the $t$-th iteration.

Traditional analysis shows that SGD approaches a stationary point of the training loss if the learning rate is set to be sufficiently small depending on the smoothness constant and noise scale. In this viewpoint, if we reduce the learning rate by a factor $10$, the end result is the same, and just takes $10$ times as many steps. SGD with very tiny step sizes can be thought of as Gradient Descent (GD) (i.e., gradient descent with full gradient), which in the limit of infinitesimal step size approaches Gradient Flow (GF). 

However, it is well-known that using only small learning rates or large batch sizes (while fixing other hype-parameters) may lead to worse generalization~\citep{bengio2012practical,keskar2016large}. From this one concludes that finite (not too small) learning rate ---alternatively, noise in the gradient estimate, or small batch sizes--- play an important role in generalization, and many authors have suggested that the noise helps avoid sharp minima~\citep{hochreiter1997flat,keskar2016large,li2018visualizing,izmailov2018averaging,he2019asym}. Formal understanding of the effect  involves modeling SGD via a Stochastic Differential Equation (SDE) in the continuous time limit~\citep{li2019stochastic}:
\[
	d\vW_t = - \eta(t) \lambda \vW_t dt - \eta(t) \nabla \Loss(\vW_t) dt + \eta(t) \mSigma_{\vW_t}^{1/2} d\vB_t,
\]
where $\mSigma_{\vw}$ is the covariance matrix of the noise at $\vW_t = \vw$. Several works have adopted this  SDE view and given some rigorous analysis of the effect of noise~\citep{smith2018a,chaudhari2018stochastic,shi2020learning}.

While this SDE view is well-established, we will note in this paper that the past works (both theory and experiments) often draw intuition from shallow nets and do not help understand modern architectures, which can be very deep and crucially rely upon  normalization schemes such as Batch Normalization (BN)~\citep{ioffe2015batch}, Group Normalization (GN)~\citep{wu2018group}, Weight Normalization (WN)~\citep{salimans2016weight}. We will discuss in \Cref{sec:incompat} that these normalization schemes are incompatible to the traditional view points in the following senses. First, normalization makes the loss provably non-smooth around origin, so GD could behave (and does behave, in  our experiments) significantly differently from its continuous counterpart, GF, if weight decay is turned on. For example, GD may oscillate between zero loss and high loss and thus cannot persistently achieve perfect interpolation. Second, there is experimental evidence suggesting that the above SDE may be far from mixing for normalized networks in normal training budgets. Lastly, assumptions about the noise in the gradient being a fixed Gaussian turn out to be unrealistic.

In this work, we incorporate effects of normalization in the SDE view to study the complex interaction between BN, weight decay, and learning rate schedule. We particularly focus on Step Decay, one of the most commonly-used learning rate schedules. Here the training process is divided into several phases $1, \dots, K$. In each phase $i$, the learning rate $\eta_t$ is kept as a constant $\bar{\eta}_i$, and the constants $\bar{\eta}_i$ are decreasing with the phase number $i$.
Our main experimental observation is the following one, which is formally stated as a conjecture in the context of SDE in \Cref{subsec:conjecture}.
\begin{observation}\label{obs:main}
	If trained for sufficiently long time during some phase $i$ ($1\le i\le K$), a neural network with BN and weight decay will eventually reach an equilibrium distribution in the function space. This equilibrium only depends on the product $\bar{\eta}_i\lambda$, and is independent of the history in the previous phases. Furthermore, the time that the neural net stays at this equilibrium will not affect its future performance.
\end{observation}

\myparagraph{Our contributions.} In this work, we identify a new ``intrinsic LR'' parameter $\lambda_e=\lambda\eta$ based on \Cref{obs:main}. The main contributions are the following:
\begin{enumerate}
	\item We theoretically analyse how intrinsic LR controls the evolution of effective speed of learning and how it leads to the equilibrium. This is done through incorporating BN and weight decay into the classical framework of Langevin Dynamics (\Cref{sec:theory}).
	\item Based on our theory, we empirically observed that small learning rates can perform equally well, which challenges the popular belief that good generalization requires large initial LR (\Cref{sec:evidence}).
	\item Finally, we make a conjecture, called \emph{Fast Equilibrium Conjecture}, based on mathematical intuition (\Cref{sec:theory}) and experimental evidence (\Cref{sec:evidence}): the number of steps for reaching equilibrium in \Cref{obs:main} scales inversely to the intrinsic LR, as opposed to the mixing time upper bound $e^{O(1/\eta)}$ for Langevin dynamics~\citep{bovier2004metastability,shi2020learning}. This gives a new perspective in understanding why BN is effective in deep learning.
\end{enumerate}

\section{Related Works}
\myparagraph{Effect of Learning Rate / Batch Size.} The generalization issue of large batch size / small learning rate has been observed as early as \citep{bengio2012practical,LeCun2012efficient}. \citep{keskar2016large}~argued that the cause is that large-batch training tends to converge to sharp minima, but \citep{dinh2017sharp} noted that sharp minima can also generalize well due to invariance in ReLU networks. \citep{li19learningrate}~theoretically analysed the effect of large learning rate in a synthetic dataset to argue that the magnitude of learning rate changes the learning order of patterns in non-homogeneous dataset. To close the generalization gap between large-batch and small-batch training, several works proposed to use a large learning rate in the large-batch training to keep the scale of gradient noise~\citep{hoffer2017train,smith2018a,chaudhari2018stochastic,smith2018dont}. \citet{shallue2019measuring}~demonstrated through a systematic empirical study that there is no simple rule for finding the optimal learning rate and batch size as the generalization error could largely depend on other training metaparameters.
None of these works have found that training without a large initial learning rate can generalize equally well in presence of BN.

\myparagraph{Batch Normalization.} Batch Normalization is proposed in~\citep{ioffe2015batch}. While the original motivation is to reduce Internal Covariate Shift (ICS), \citep{santurkar2018does} challenged this view and argued that the effectiveness of BN comes from a smoothening effect on the training objective. \citep{bjorck2018understanding} empirically observed that the higher learning rate enabled by BN is responsible for the better generalization. \citep{kohler2019exp} studied the direction-length decoupling effect in BN and designed a new optimization algorithm with faster convergence for learning 1-layer or 2-layer neural nets. Another line of works focus on the effect of scale-invariance induced by BN as well as other normalization schemes. \citep{hoffer2018norm} observed that the effective learning rate of the parameter direction is $\frac{\eta}{\norm{\vw_t}^2}$. \citep{arora2018theoretical}~identified an auto-tuning behavior for the effective learning rate and \citep{cai2019aquantitative}~gave a more quantitative analysis for linear models. In presence of BN and weight decay, \citep{van2017l2}~showed that the gradient noise causes the norm to grow and the weight decay causes to shrink, and the effective learning rate eventually reaches a constant value if the noise scale stays constant. \citep{zhang2018three}~validated this phenomenon in the experiments. \citep{li2020exp} rigorously proved that weight decay is equivalent to an exponential learning rate schedule.

\section{Preliminaries} \label{sec:bg}

\myparagraph{Stochastic Gradient Descent and Weight Decay.} Let $\{ (\vx_i, y_i) \}_{i=1}^{n}$ be a dataset consisting of input-label pairs. In (mini-batch) stochastic gradient descent, the following is the update, where  $\Batch_t \subseteq \{1, \dots, n\}$ is a mini-batch of random samples, $\vw_t$ is the vector of trainable parameters of a neural network, $\Loss(\vw; \Batch) = \frac{1}{\lvert \Batch \rvert}\sum_{b \in \Batch} \ell_\Batch(\vw; \vx_b, y_b)$ is the average mini-batch loss  (we use subscript $\Batch$ because $\ell$ can depend on $\Batch$ if BN is used) and $\eta_t$ is the learning rate (LR) at step $t$:
\begin{equation} \label{eq:sgd-original}
	\vw_{t+1} \gets \vw_t - \eta_t \nabla \Loss(\vw_t; \Batch_t).
\end{equation}
{\em Weight decay (WD)} with parameter $\lambda$ (a.k.a., adding an $\ell_2$ regularizer term $\frac{\lambda}{2} \normtwo{\vw}^2$ ) is standard in networks with BN, yielding the update:
\begin{equation} \label{eq:sgd-wd}
	\vw_{t+1} \gets (1 - \eta_t\lambda)\vw_t - \eta_t \nabla \Loss(\vw_t; \Batch_t).
\end{equation}
\myparagraph{Normalization Schemes and Scale-invariance.} Batch normalization (BN)~\citep{ioffe2015batch} makes the training loss invariant to re-scaling of layer weights, as it normalizes the output for every neuron (see Appendix~\ref{app:bn} for details; scale-invariance emerges if the output layer is fixed). We name this property as \emph{scale-invariance}. More formally, we say a function $f: \R^d \to \R$ is \emph{scale-invariant} if $f(\vw) = f(\alpha\vw), \forall \vw \in \R^d,\alpha>0$. Note that scale-invariance is a general property that also holds for loss in presence of other normalization schemes~\citep{wu2018group,salimans2016weight,ba2016layer}. 

Scale-invariance 
implies the gradient and Hessian are inversely proportional to $\norm{\vw}, \norm{\vw}^2$ respectively, meaning that the smoothness is unbounded near $\vw = 0$. This can be seen by taking gradients with respect to $\vw$ on both sides of $ f(\vw) = f(\alpha\vw)$:
\begin{lemma}\label{lem:inv-grad-hess}
    For a scale-invariant function $f: \R^d \to \R$, $\nabla f(\alpha \vw) = \frac{1}{\alpha}\nabla f(\vw)$ and $\nabla^2 f(\alpha \vw) = \frac{1}{\alpha^2}\nabla^2 f(\vw)$ hold for all $\vw$ and $\alpha>0$.
\end{lemma}
The gradient can also be proved to be perpendicular to $\vw$, that is, $\dotp{\nabla  f(\vw)}{\vw} = 0$ holds for all $\vw$.
This property can also be seen as a corollary of Euler's Homogeneous Function Theorem. In the deep learning literature, \citep{arora2018theoretical} used this in the analysis of the auto-tuning behavior of normalization schemes.

\myparagraph{Approximating SGD by SDE.} Define the expected loss $\Loss(\vw) := \E_{\Batch}\left[ \Loss(\vw; \Batch) \right]$ and the error term $\vxi := \nabla \Loss(\vw; \Batch_t) - \nabla \Loss(\vw)$. Then we can rewrite the formula of SGD with constant LR $\eta$ as
$\vw_{t+1} \gets \vw_t - \eta \left(\nabla \Loss(\vw_t) + \vxi_t \right)$.
The mean of gradient noise is always $0$. The covariance matrix of the gradient noise at $\vw$ equals to $\mSigma_{\vw} := \E_{\Batch}[(\nabla \Loss(\vw; \Batch) - \nabla \Loss(\vw))(\nabla \Loss(\vw; \Batch) - \nabla \Loss(\vw))^{\top}]$. To approximate SGD by SDE, the classic approach is to model the gradient noise by Gaussian noise $\vxi_t \sim \mathcal{N}(\vzero, \mSigma_{\vw_t})$, and then take the continuous time limit to obtain the surrogate SDE for infinitesimal LR \citep{li2017stochastic,cheng2019stochastic}. As is done in previous works~\citep{smith2018a,smith2018dont,chaudhari2018stochastic,shi2020learning}, we also use this surrogate dynamics to approximate SGD with LR of any size:
\begin{equation} \label{eq:sgd-sde}
	\dd \vW_{t} = - \eta \left(\nabla \Loss(\vW_t) \dd t+ (\mSigma_{\vW_t})^{\frac{1}{2}} \dd \vB_t\right).
\end{equation}
Here $\vB_t \in \R^d$ is the Wiener Process (Brownian motion), which satisfies $\vB_t - \vB_s \sim N(\vzero, (t-s) \mI_d)$ conditioned on $\vB_s$. When $ \mSigma_\vw $ is $\vzero$ (the full-batch GD case), \eqref{eq:sgd-sde} is known as \emph{gradient flow}.

\myparagraph{Folklore view of landscape exploration.} There is evidence that the training loss has many global minima (or near-minima), whose test loss values can differ radically. The basins around these global minima are separated by \textquotedblleft hills\textquotedblright\ and only large noise can let SGD jump from one to another, while small noise will only make the network oscillate in the same basin around a minimum. The regularization effect of large LR/large noise happens because (1) sharp minima have worse generalization  (2) noise prevents getting into narrow basins and thus biases exploration towards flatter basins. (But this view is known to be simplistic, as noted in many papers.)

\section{Apparent Incompatibility between BN and Traditional View Points} \label{sec:incompat}
In this section, we discuss how BN leads to issues with the traditional optimization view of gradient flow and SDE. This motivates our new view in Section~\ref{sec:theory}.

\myparagraph{Full batch gradient descent $\neq$ gradient flow.} It's well known that if LR is smaller than the inverse of the smoothness, then  trajectory of gradient descent will be close to that of gradient flow. But for normalized networks, the loss function is scale-invariant and thus provably non-smooth (i.e., smoothness becomes unbounded)  around origin  \citep{li19learningrate}. (By contrast, without WD, the SGD moves away from origin~\citep{arora2018theoretical} since norm increases monotonically.) We will show that this nonsmoothness is very real and makes training unstable and even chaotic for full batch SGD with any nonzero learning rate. And yet convergence of gradient flow  is unaffected.

Consider a toy scale-invariant loss, $L(x,y) = \frac{x^2}{x^2+y^2}$. Since loss only depends on $x/y$,  WD has no effect on it.  Even with WD turned on, Gradient Flow (i.e., infinitesimal updates) will lead to monotone decrease in $|x_t/y_t|$. But \Cref{subfig:chaotic_toy} in the appendix shows that dynamics  for GD with WD are chaotic:  as similar trajectories approach the origin, tiny differences  are amplified and they diverge.

Modern deep nets with BN + WD (the standard setup) also exhibit instability close to zero loss. See \Cref{subfig:chaotic_gd_acc,subfig:chaotic_gd_norm}, where deep nets being trained on small datasets exhibit oscillation between zero loss and high loss. In any significant period with low loss (i.e., almost full accuracy), gradient is small but WD continues to reduce the norm, and resulting non-smoothness leads to large increase in loss.

\begin{figure}[!htbp]
	\centering
	\begin{subfigure}[b]{0.29\textwidth}
		\centering
		\includegraphics[width=\linewidth]{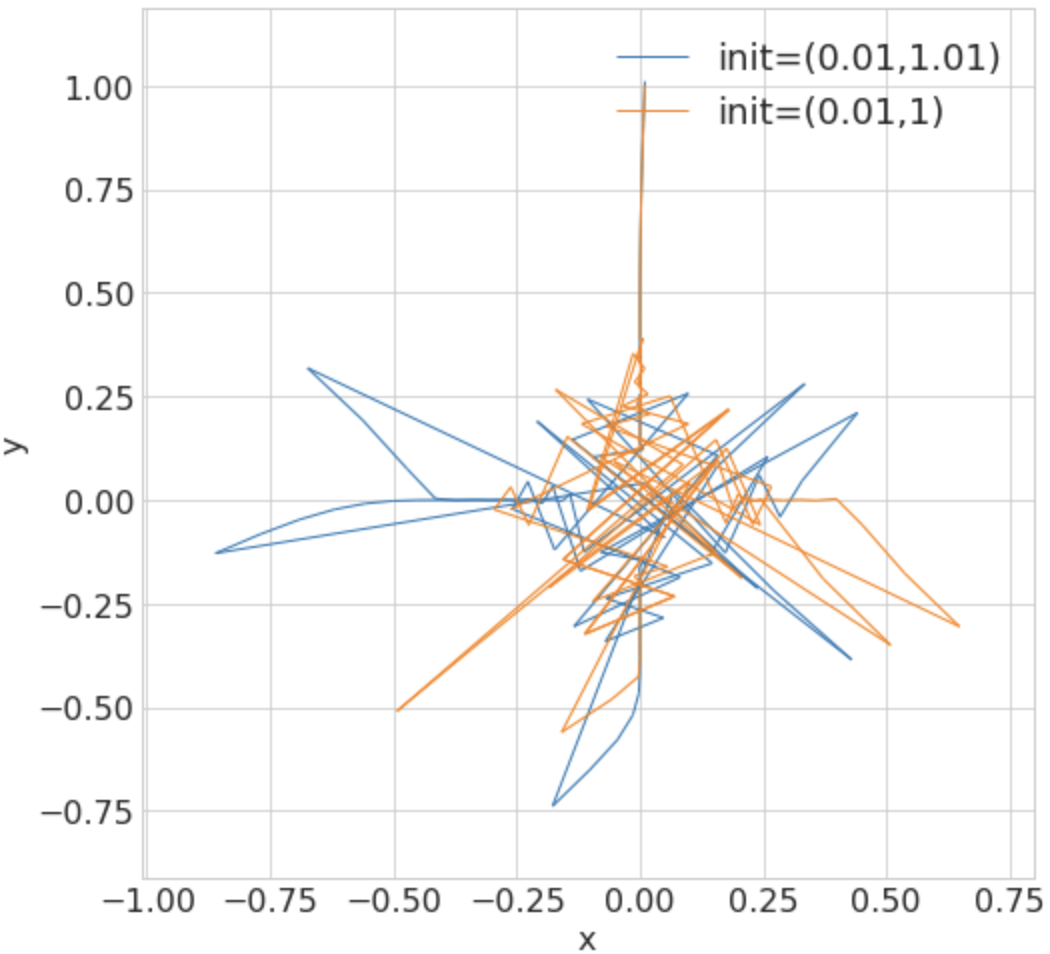}
		\caption{GD with WD on $L(x,y)$}
		\label{subfig:chaotic_toy}
	\end{subfigure}%
	\begin{subfigure}[b]{0.35\textwidth}
		\centering
		\includegraphics[width=\linewidth]{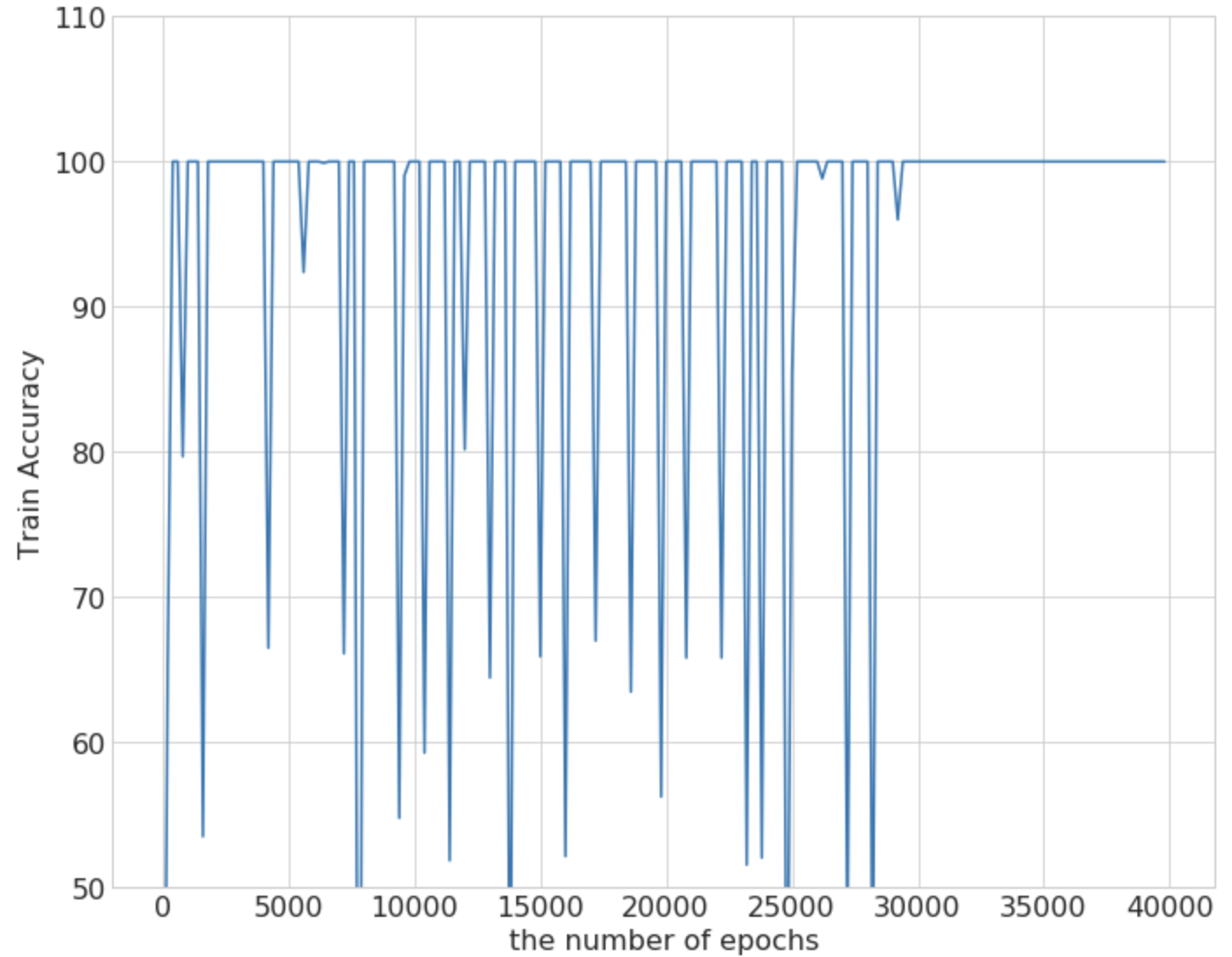}
		\caption{Train accuracy}
		\label{subfig:chaotic_gd_acc}
	\end{subfigure}%
	\begin{subfigure}[b]{0.35\textwidth}
		\centering
		\includegraphics[width=\linewidth]{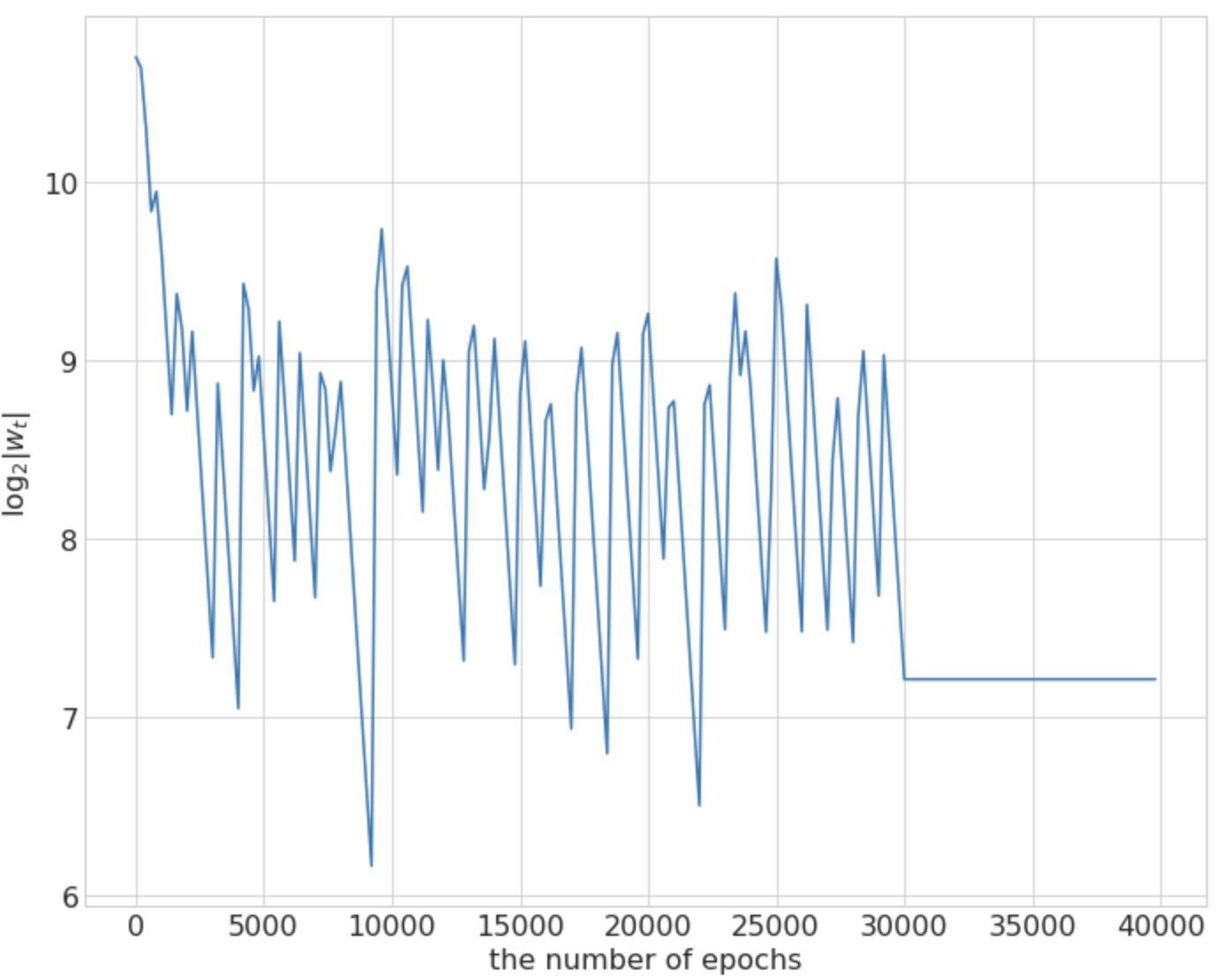}
		\caption{Norm}
		\label{subfig:chaotic_gd_norm}
	\end{subfigure}%
	\caption{\small WD makes GD on scale-invariant loss unstable and chaotic, for both the toy model and PreResNet32. \textbf{(a)} The toy model is trained GD with LR $0.1$, WD $0.5$ and two initial points near zero loss. The initial points are very close to each other. \textbf{(b)(c)} \emph{Convergence never truly happens} for PreResNet32 trained on sub-sampled CIFAR10 containing 1000 images with full-batch GD, WD $5 \times 10^{-4}$, LR $1.6$ (without momentum). PreResNet32 can easily get $100\%$ training accuracy but is unable to stay long. WD is turned off at epoch 30000.}
	\label{fig:chaotic}
\end{figure}

\begin{figure}[!htbp]
        \centering
        \begin{subfigure}[b]{0.57\textwidth}
                \centering
				\includegraphics[width=0.7\linewidth]{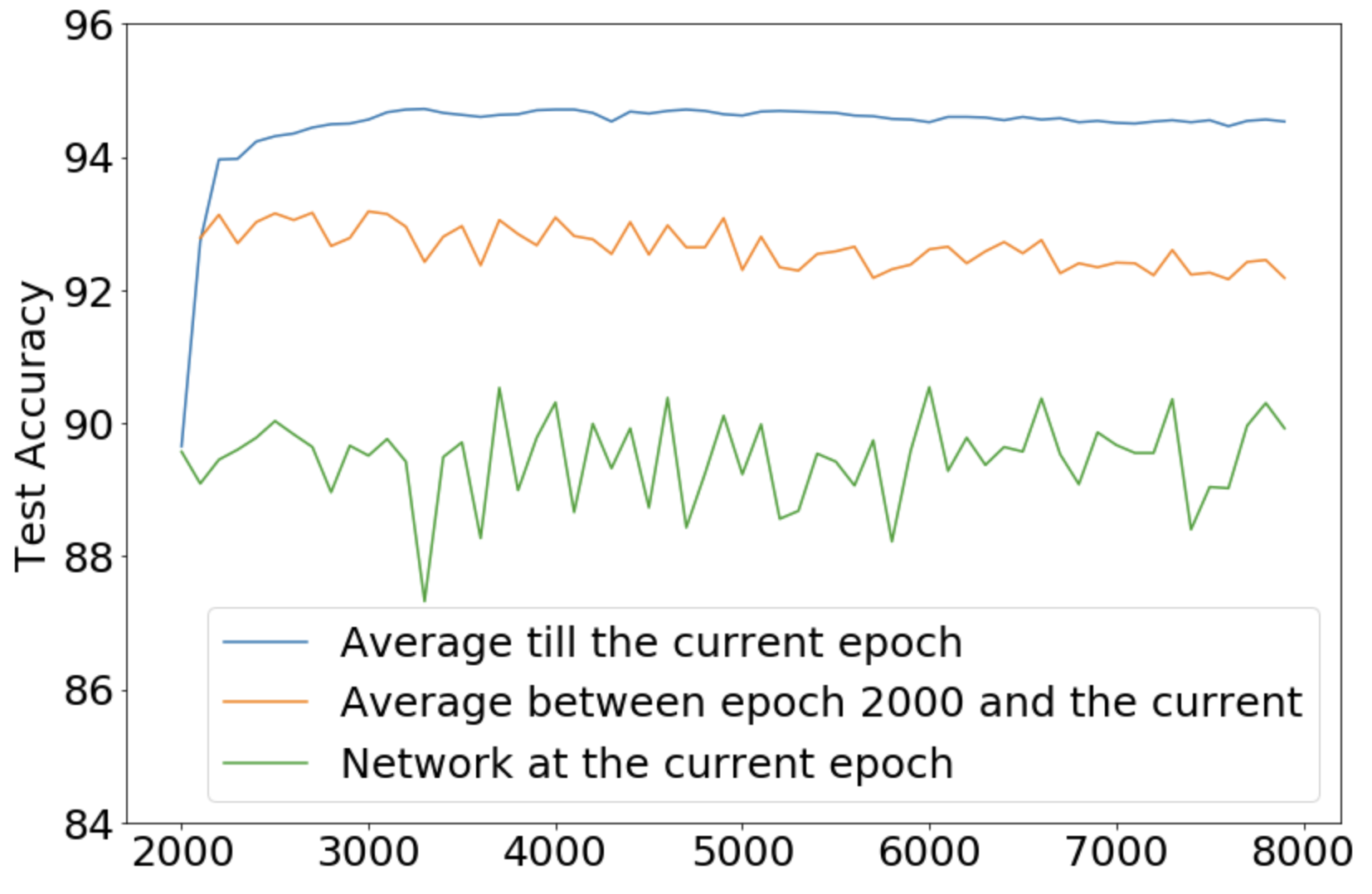}
				\caption{Test accuracy} \label{fig:swa_acc}
        \end{subfigure}%
        \begin{subfigure}[b]{0.42\textwidth}
                \centering
                \includegraphics[width=0.7\linewidth]{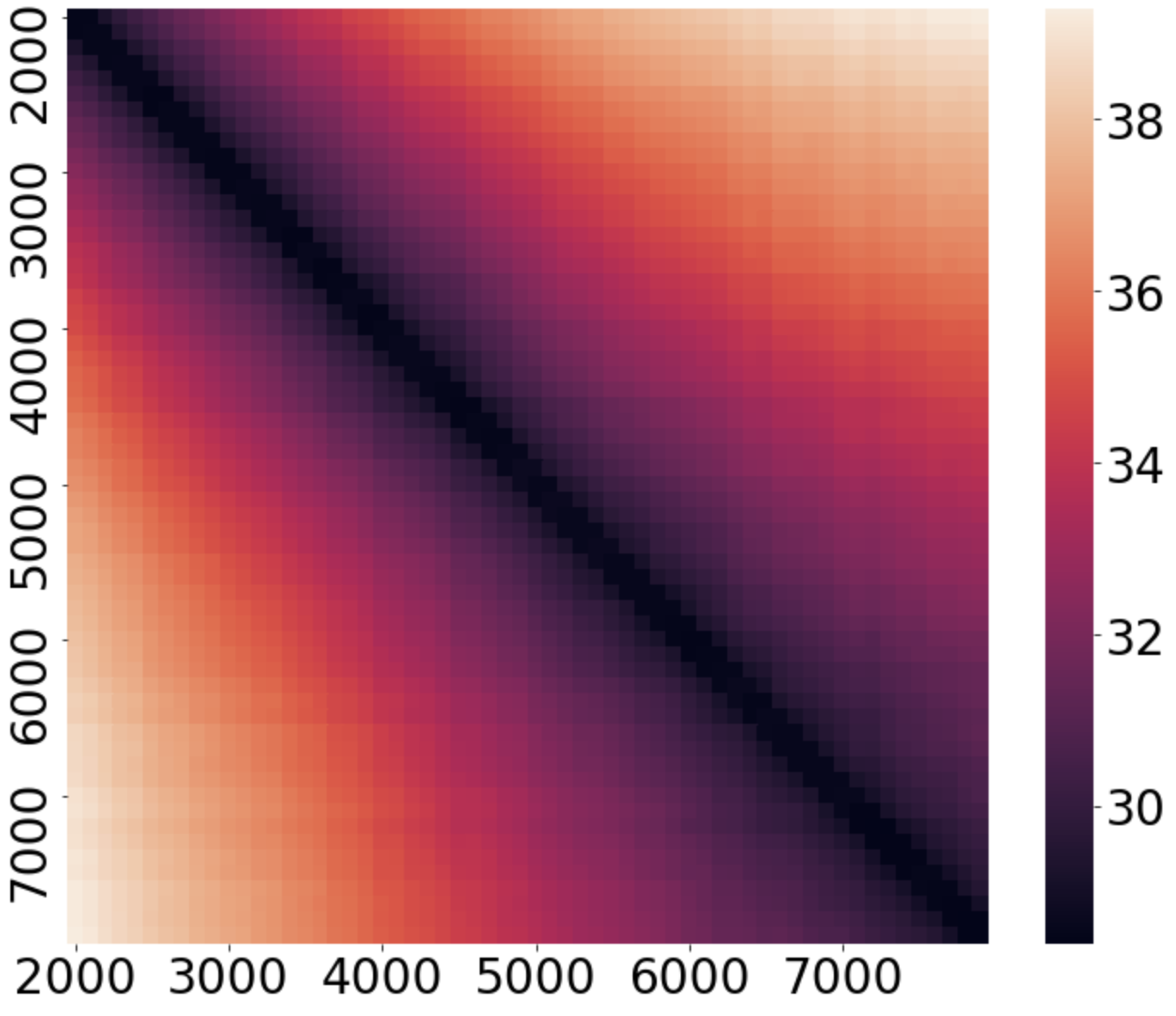}
                \caption{ Pairwise $\ell_2$ distance} \label{fig:swa_dist}
        \end{subfigure}%
        \caption{\small Stochastic Weight Averaging improves the test accuracy of PreResNet32 trained with momentum SGD 
        on CIFAR10. In particular, test accuracy is improved by $2\% $ by simply averaging a network with any other network along the same trajectory, suggesting the trajectory is still local.
        However, the distance between parameters keeps increasing. As a comparison, the average parameter norm (over epoch 2000-8000) is around 39, which has exactly the same magnitude as the pairwise distance, indicating the langevin diffusion view  around strongly convex local optimum in \citep{izmailov2018averaging} may not suffice to explain the success of SWA.}
        \label{fig:swa}
\end{figure}

\myparagraph{Problems with random walk/SDE view of SGD.} The standard story about the role of noise in deep learning is that it turns a deterministic process into a geometric random walk in the landscape, which can in principle explore the landscape more thoroughly, for instance by occasionally making loss-increasing steps. Rigorous analysis of this walk is difficult since the mathematics of real-life training losses is not understood. But assuming the noise in SDE is a fixed Gaussian,  the stationary distribution of the random walk can be shown to be the familiar Gibbs distribution over the landscape. See~\citep{shi2020learning} for a recent account, where SDE is shown to converge to equilibrium distribution in $e^{O(C/\eta)}$ time for some term $C$ depending upon loss function. This convergence is extremely slow for small LR $\eta$ and thus way beyond normal training budget.

Recent experiments have also suggested the walk does not reach this initialization-independent equilibrium within normal training time. Stochastic Weight Averaging (SWA) \citep{izmailov2018averaging} shows that the loss landscape is nearly convex along the trajectory of SGD with a fixed hyper-parameter choice, e.g., if  the two network parameters from different epochs are averaged,  the test loss is lower. This reduction can go on for 10 times more than the normal training budget as shown in~\Cref{fig:swa}.  However, the accuracy improvement is a very local phenomenon since it doesn't happen for SWA between solutions obtained from different initialization, as shown in~\citep{draxler2018essentially,garipov2018loss}. This suggests the networks found by SGD within normal training budget highly depends on the initialization, and thus SGD doesn't mix in the parameter space.

Another popular view (e.g., \citep{izmailov2018averaging}) believes that instead of mixing to the unique global equilibrium, the trajectory of SGD could be well approximated by a multivariate Ornstein-Uhlenbeck (OU) process around a local minimizer $\mW_{*}$, assuming the loss surface is locally strongly convex. As the corresponding stationary point is a Gaussian distribution $\mathcal{N}(\mW_{*}, \mSigma)$, this explains why SWA helps to reduce the training loss. However, this view is challenged by the fact that the $\ell_2$ distance between weights from epochs $T$ and $T+\Delta$ monotonically increases with $\Delta$ for every $T$ (See \Cref{fig:swa_dist}), while $\E[\normtwosm{\mW_T-\mW_{T+\Delta}}^2]$ should converge to the constant $2 \Tr(\mSigma)$ as $T, \Delta \to +\infty$ in the OU process. This suggests that all these weights are correlated and haven't mixed to Gaussian.

For the case where WD is turned off, \citep{arora2018theoretical} proves that the norm of weight is monotone increasing, thus the mixing in parameter space provably doesn't exist for SGD with BN.

\section{SDE-based framework for modeling SGD on Normalized Networks} \label{sec:theory}
For SGD with learning rate $\eta$ and weight decay $\lambda$,  we define $\lambda_e := \eta\lambda$ to be the \textit{effective weight decay}. This is actually  the original definition of weight decay~\citep{hanson1989comparing} and is also proposed (based upon experiments) in \citep{loshchilov2018decoupled} as a way to improve generalization for Adam and SGD.
In Section \ref{subsec:SDE_WD_Norm}, we will suggest calling $\lambda_e$ the \emph{intrinsic learning rate} because it controls
trajectory in a manner similar to learning rate. Now we can rewrite update rule \eqref{eq:sgd-wd} and its corresponding SDE as
\begin{align} 
\vw_{t+1} &\gets (1 - \lambda_e)\vw_t - \eta \left(\nabla \Loss(\vw_t) + \vxi_t \right). \label{eq:sgd-wd-noisy} \\
\dd \vW_{t} &= - \eta \left(\nabla \Loss(\vW_t) \dd t+ (\mSigma_{\vW_t})^{\frac{1}{2}} \dd \vB_t\right) - \lambda_e \vW_t \dd t. \label{eq:sde-wd}
\end{align}

\subsection{SDE with Weight Decay and Normalization}\label{subsec:SDE_WD_Norm}
When the loss function is scale-invariant, the gradient noise $\mSigma_{\vW}$ is inherently anisotropic and position-dependent: Lemma~\ref{lem:perp_cov} in the appendix shows the noise lies in the subspace perpendicular to $w$ and blows up close to the origin. To get an SDE description closer to the canonical format, we reparametrize parameters to unit norm.
Define $\hW_t = \frac{\vW_t}{\norm{\vW_t}}$, $G_t = \norm{\vW_t}^2$, where $\norm{\vw}$ stands for the $\normltwo$-norm of a vector $\vw$.
The following Lemma is proved in the appendix using It\^{o}'s Lemma:
\begin{restatable}{theorem}{mainlemma} \label{lm:evolution}
	The evolution of the system can be described as:
\begin{align}
\dd \hW_{t} &= - \frac{\eta}{G_t} \left(\nabla \Loss(\hW_t)\dd t + (\mSigma_{\hW_t})^{\frac{1}{2}} \dd \mB_t\right) - \frac{\eta^2}{2G_t^2}  \Tr(\mSigma_{\hW_t}) \hW_t\dd t \label{eq:sde-hW} \\ \vspace{-0.1cm}
\frac{\dd G_t}{\dd t} &= -2 \lambda_e G_t + \frac{\eta^2}{G_t}\Tr(\mSigma_{\hW_t}) \label{eq:sde-rho}
\end{align}
\vspace{-0.1cm}
where $\Tr(\mSigma_{\vw})$ is the trace of the covariance matrix of gradient noise at $\vw$.\footnote{\scriptsize $\Tr(\mSigma_{\vw}) = \E\left[\norm{\nabla \Loss(\vw; \Batch) - \nabla \Loss(\vw)}^2\right]$, so it is also a noise term.}
\end{restatable}
The SDE enables clean mathematical demonstration of many properties of normalization schemes. 
For example, dividing both sides of \eqref{eq:sde-rho} by $\eta$ gives
\begin{equation}
\frac{\dd(G_t/\eta)}{\dd t} = -2 \lambda_e \cdot \frac{G_t}{\eta} + \frac{\eta}{G_t}\Tr(\mSigma_{\hW_t}).
\end{equation}
This shows that the dynamics only depends on the ratio $G_t/\eta$, which also suggests that {\em initial LR is of limited importance, indistinguishable from scale of initialization.}  Now define $\gamma_t := (G_t/\eta)^2$. ($\eta/G_t = \gamma_t^{-0.5}$ was called the \emph{effective learning rate} in \citep{hoffer2018norm,zhang2018three,arora2018theoretical}.) This simplifies the equations:
\begin{align}
\dd \hW_{t} &= - \gamma_t^{-1/2} \left(\nabla L(\hW_t)\dd t + (\mSigma_{\hW_t})^{\frac{1}{2}} \dd \mB_t\right) - \frac{1}{2\gamma_t}  \Tr(\mSigma_{\hW_t}) \hW_t\dd t. \label{eq:gamma_W} \\
\frac{\dd \gamma_t}{\dd t} &= -4 \lambda_e \gamma_t + 2 \Tr(\mSigma_{\hW_t}) \label{eq:gamma}.
\end{align}
\eqref{eq:gamma} can be alternatively written as the following, which shows that squared effective LR $\gamma_t$ is a running average of the norm squared of gradient noise.
\begin{equation} \label{eq:gamma-avg}
	\gamma_t = e^{-4 \lambda_et} \gamma_0 + 2 \int_{0}^{t} e^{-4 \lambda_e (t-\tau)} \Tr(\mSigma_{\hW_{\tau}}) d\tau.
\end{equation}
Experimentally\footnote{\scriptsize \Cref{fig:7decay,fig:vgg_fast_conjecture} shows that after a certain length of time the relationship $\gamma_t^{1/2} \propto \lambda_e^{-1/2}$ holds approximately, up to a small multiplicative constant. Since $\gamma_t$ is the running average of $\Tr(\Sigma_{\hW})$, the magnitude of the noise, it suggests for different regions of the landscape explored by SGD with different intrinsic LR $\lambda_e$, the noise scales don't differ a lot.} we find that the trace of noise is approximately constant. This is the assumption of the next lemma (much weaker than assumption of fixed gaussian noise in past works). 
\begin{lemma} \label{lm:gamma-decay}
If $\sigma^2 \le \Tr(\Sigma_{\hW}) \le (1+\eps)\sigma^2$ for all $\hW$ encountered in the trajectory, then
\begin{equation}\label{eqn:norm_mixing2}
	\gamma_t = e^{-4 \lambda_et} \gamma_0 + \left(1 + O(\epsilon)\right) \frac{\sigma^2}{2\lambda_e}\left( 1 - e^{-4\lambda_e t}\right).
\end{equation}
\end{lemma}

The lemma  again suggests that the initial effective LR decided together by LR $\eta$ and norm $\norm{\vW_0}$ only has a temporary effect on the dynamics: no matter how large is the initial effective LR, after $O(1/\lambda_e)$ time, the effective LR $\gamma_t^{-1/2}$ always converges to the stationary value $(1 + O(\epsilon)) \frac{\sigma}{\sqrt{2 \lambda_e}} \propto \lambda_e^{-1/2}$.

\subsection{A conjecture about mixing time in function space}\label{subsec:conjecture}
As mentioned earlier, there is evidence that SGD may take a very long time to mix in the parameter space. However, we observed that test/train errors converge in expectation soon after the norm $\norm{\vW_t}$ converges in expectation, which only takes $O(1/\lambda_e)$ time by our theoretical analysis. More specifically, we find experimentally that the number of steps of SGD (or the length of time for SDE) after the norm converges doesn't significantly affect the expectation of any known statistics related to the training procedure, including train/test errors, and even the output distribution on every single datapoint.
This suggests the neural net reaches an ``equilibrium state'' in function space.

The above findings motivate us to define the notion of ``equilibrium state'' rigorously and make a conjecture formally. For learning rate schedule $\eta(t)$ and effective weight decay schedule $\lambda_e(t)$, we define $\nu(\mu; \lambda_{e},\eta, t)$ to be the marginal distribution of $\vW_t$ in the following SDE when $\vW_0 \sim \mu$:
\begin{equation} \label{eq:SDE_wd}
\dd \vW_{t} = - \eta(t) \left(\nabla \Loss(\vW_t)\dd t + (\mSigma_{\vW_t})^{\frac{1}{2}} \dd \vB_t\right) - \lambda_e(t) \vW_t\dd t.
\end{equation}
For a random variable $X$, we define $P_{X}$ to be the probability distribution of $X$. The total variation $\dTV(P_1, P_2)$ between two probability measures $P_1, P_2$ is defined by the supremum of $\lvert P_1(A) - P_2(A) \rvert$ over all measurable set $A$. Given input $\vx$ and neural net parameter $\vw$, we use $F(\vw; \vx)$ to denote the output of the neural net (which is the pre-softmax logits in classification tasks).
\begin{conjecture}[Fast Equilibrium Conjecture] \label{conj:weaker}
	Under the dynamics of \eqref{eq:gamma_W} and \eqref{eq:gamma}, modern neural nets converge to the equilibrium distribution in $O(1/\lambda_e)$ time in the following sense. Given two initial distributions $\mu, \mu'$ for $\vW_0$, constant learning rate and effective weight decay schedules $\lambda_e^*, \eta^*$, there exists a mixing time $T = O(1 / \lambda^*_e)$ \footnote{Here we assume the both initial weight norm and the initial LR are of constant magnitude. Otherwise there will be a multiplicative $\log\gamma_0$ factor in the mixing time, as indicated by \Cref{eq:gamma-avg}. This $\log$ dependency can usually be ignored in practice, unless the initial weight norm or LR are extremely large or small. See \Cref{fig:unfavorable_init}.} such that for any input data $\vx$ from some input domain $\mathcal{X}$,
	$\dTV\left(P_{F\left(\vW_t;~\vx\right)}, P_{F\left(\vW'_t;~\vx\right)} \right) \approx 0$ for all $t \ge T$, where $\vW_t \sim \nu(\mu; \lambda_e^*, \eta^*, t)$, $\vW'_t \sim \nu(\mu'; \lambda_e^*, \eta^*, t)$.
\end{conjecture}
It is worth to note that this conjecture obviously does not hold for some pathological initial distributions, e.g., all the neurons are initially dead. But we can verify that our conjecture holds for many initial distributions that can occur in training neural nets, e.g., random initialization, or the distribution after training with certain schedule for certain number of epochs. It remains a future work to theoretically identify the specific condition for this conjecture.

Interesting, we empirically find that the above conjecture still holds even if we are allowed to fine-tune the model before producing the output. This can be modeled by starting another SDE from $t \ge T$:
\begin{conjecture}[Fast Equilibrium Conjecture, Strong Form]
	Let  $\tilde{\eta}(\tau), \tilde{\lambda}_e(\tau)$ be a pair of learning rate and effective weight decay schedules. Under the same conditions of \Cref{conj:weaker}, there exists a mixing time $T = O(1 / \lambda^*_e)$ such that for any input data $\vx$ from some input domain $\mathcal{X}$, $\dTV\left(P_{F\left(\vW_{t, \tau};~\vx\right)}, P_{F\left(\vW_{t, \tau}';~\vx\right)} \right) \approx 0$ for all $t \ge T$, where $\vW_{t,\tau} \sim \nu\left(\nu(\mu; \lambda_e^*, \eta^*, t); \tilde{\lambda}_e, \tilde{\eta}, t\right)$, $\vW'_{t, \tau} \sim \nu\left(\nu(\mu'; \lambda_e^*, \eta^*, t); \tilde{\lambda}_e, \tilde{\eta}, t\right)$.
\end{conjecture}
In the appendix we provide the discrete version of the above conjecture by viewing each step of SGD as one step transition of a Markov Chain. By this means we can also extend the conjecture to SGD with momentum and even Adam. 	

\subsection{What happens in real life training -- An interpretation} \label{sec:twop}
Let's first recap  Step Decay -- there are $K$ phases and LR in  phase $i$ is $\bar{\eta}_i$. Below we will  explain or give better interpretation for some phenomena in real life training related to Step Decay. 

\myparagraph{Sudden increase of test error and training loss after every LR decay:} Usually here LR is dropped by something like a factor $10$. As shown above, the instant effect is to reduce effective LR  by a factor of $10$, but it gradually equilibriates to the value $\lambda_e^{-1/2}$, which is only reduced by a factor of $\sqrt{10}$. Hence there is a slow rise in error after every drop, as observed in previous works~\citep{zagoruyko16wide,zhang2018three,li2020exp}. This rise could be beneficial since it coincides with equilibrium state in function space.

\myparagraph{Intrisic LR and the final LR decay step:} However, the final LR decay needs to be treated differently. It is customary do early stopping, that finish very soon  after the final LR decay, when accuracy is best. The above paragraph can help to explain this by decomposing the training after the final LR decay into two stage. In the first stage, the effective LR is very small, so the dynamics is  closer to the classical gradient flow approximation, which can settle into a local basin. In the second stage, the effective LR increases to the stationary value and brings larger noise and worse performance. This decomposition also applies  to earlier LR decay operations, but the phenomenon is more significant for the final LR decay because the convergence time $O(1/\lambda_e)$ is much longer.

 \begin{figure}[t]
 	\centering
 	\begin{subfigure}[b]{0.49\textwidth}
 		\centering
 		\includegraphics[width=0.8\linewidth]{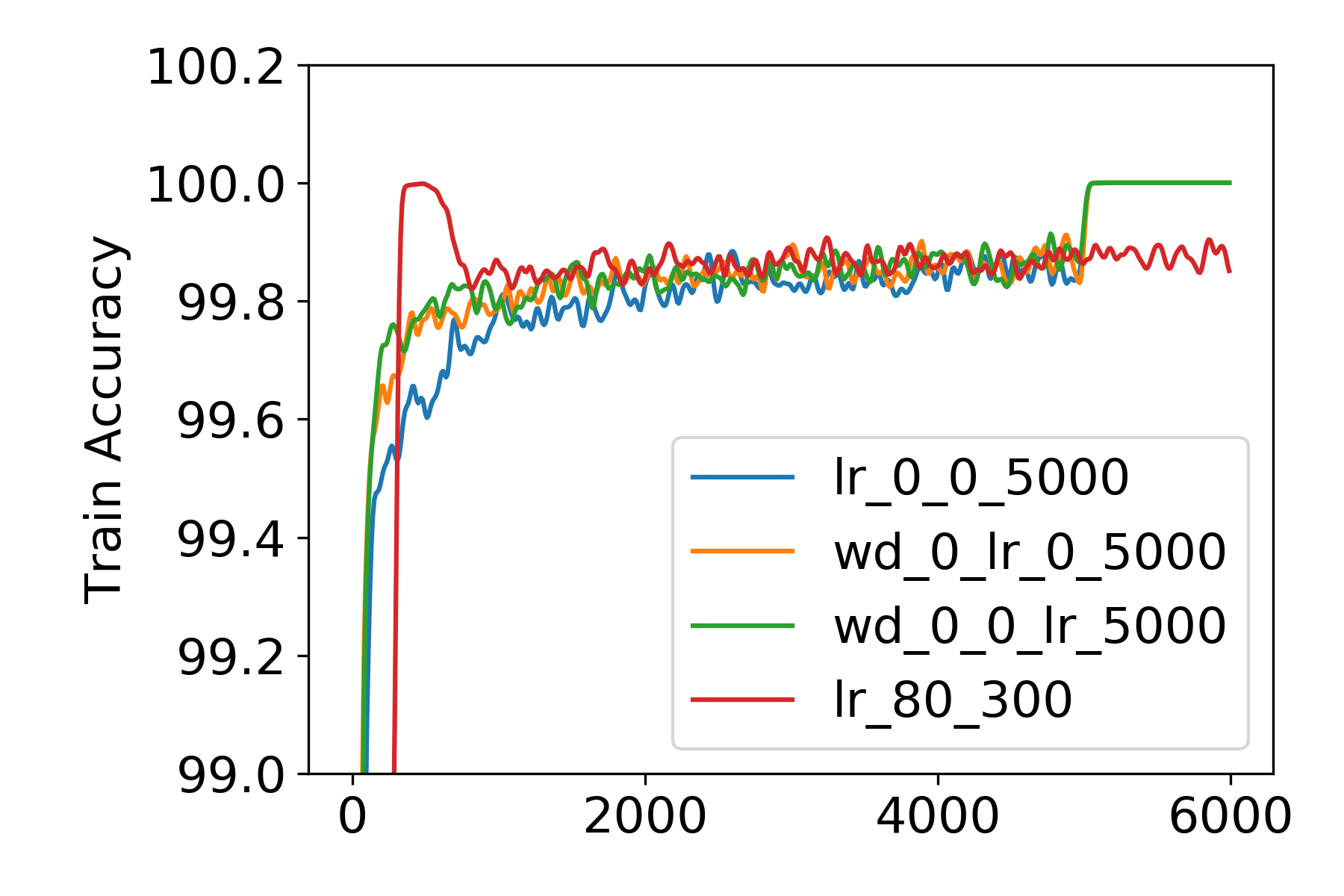}
 		\caption{Train accuracy}
 	\end{subfigure}%
 	\begin{subfigure}[b]{0.49\textwidth}
 		\centering
 		\includegraphics[width=0.8\linewidth]{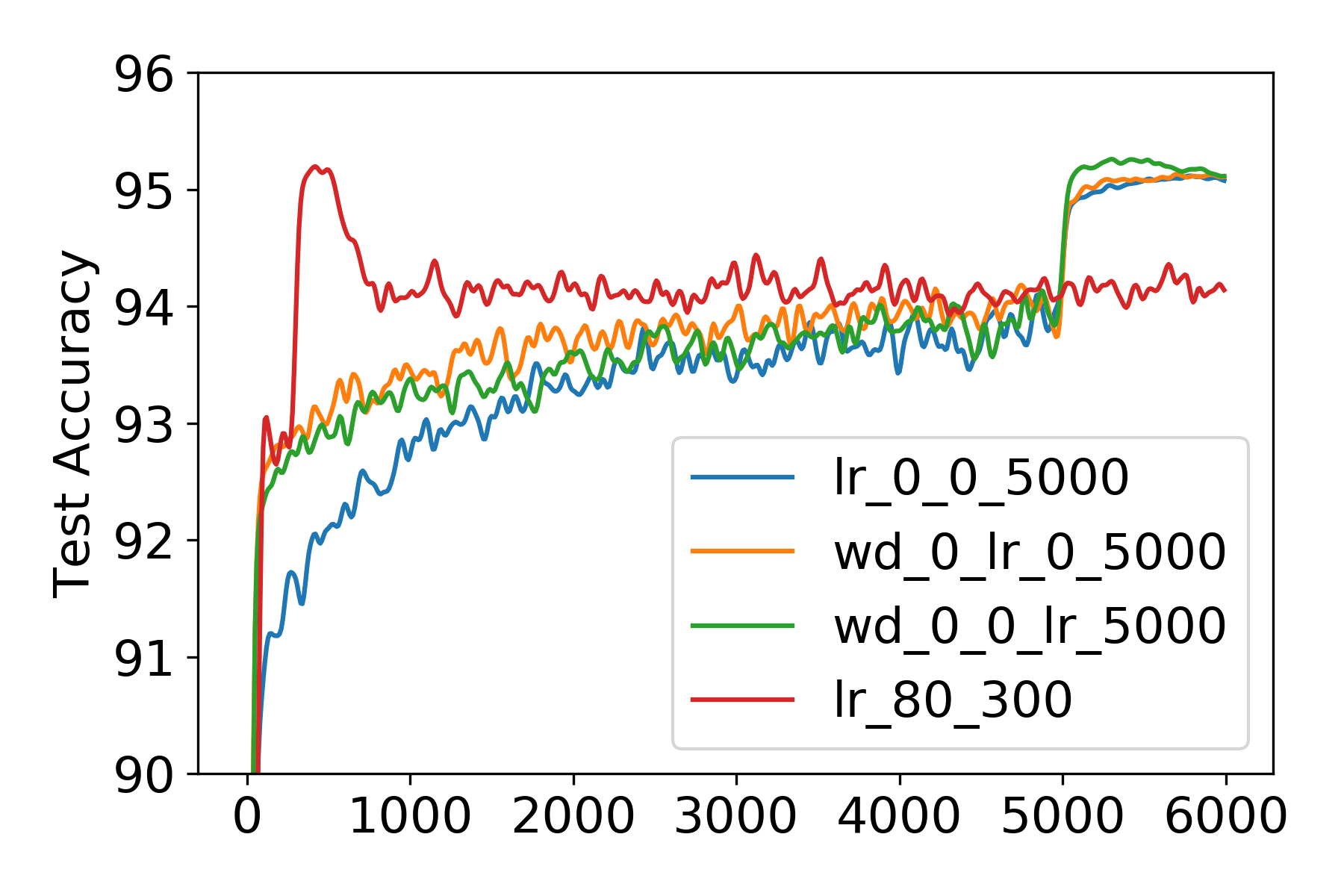}
 		\caption{Test accuracy}
 	\end{subfigure}%
 	
 	\begin{subfigure}[t]{0.49\textwidth}
 		\centering
 		\includegraphics[width=0.8\linewidth]{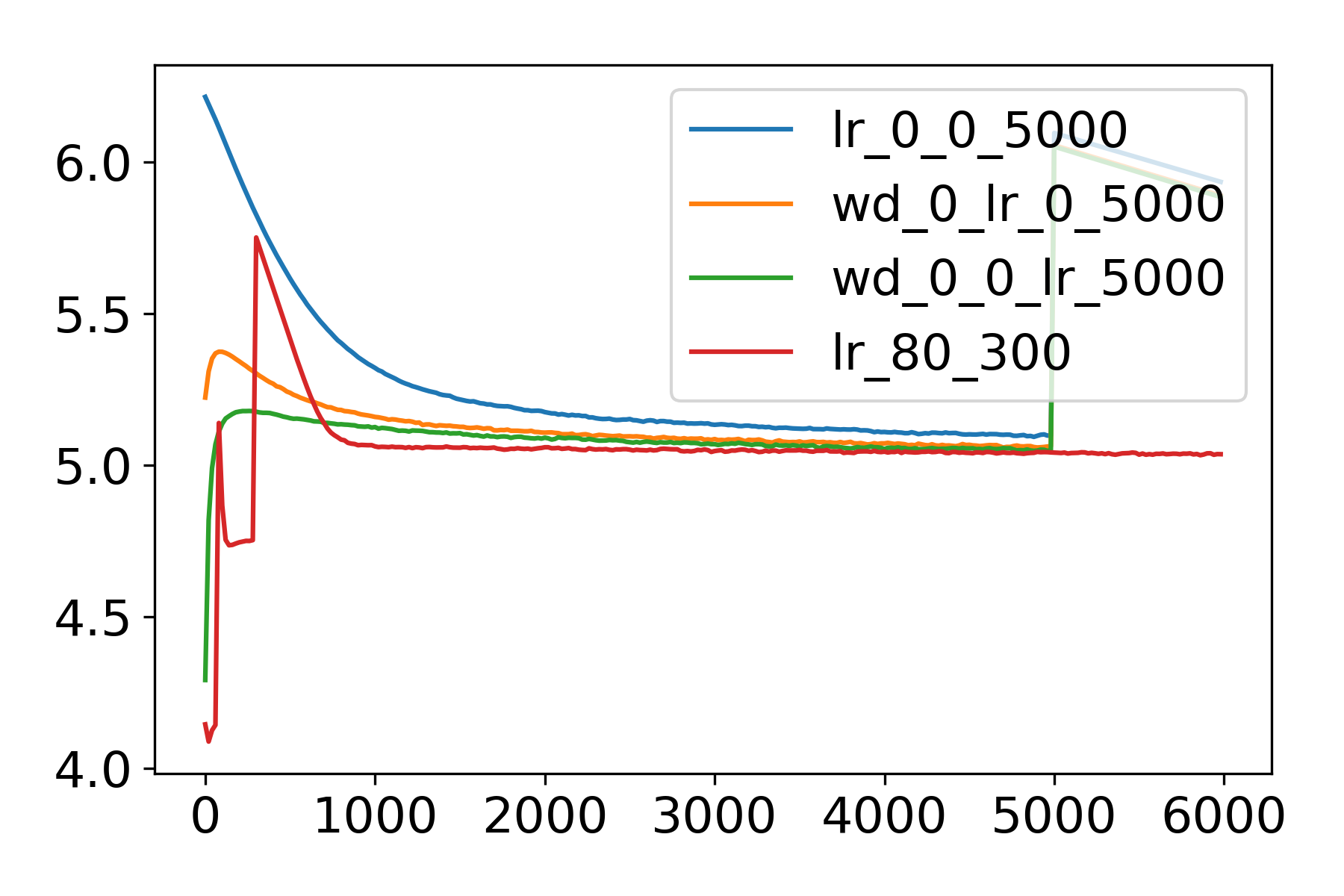}
 		\caption{$\log_{10} \gamma_t^{1/2}$ (effective LR is $\gamma_t^{-1/2}$ )}
 		\label{fig:small_lr_sota_norm}
 	\end{subfigure}%
 	\begin{subfigure}[t]{0.49\textwidth}
 		\centering
 		\includegraphics[width=0.8\linewidth]{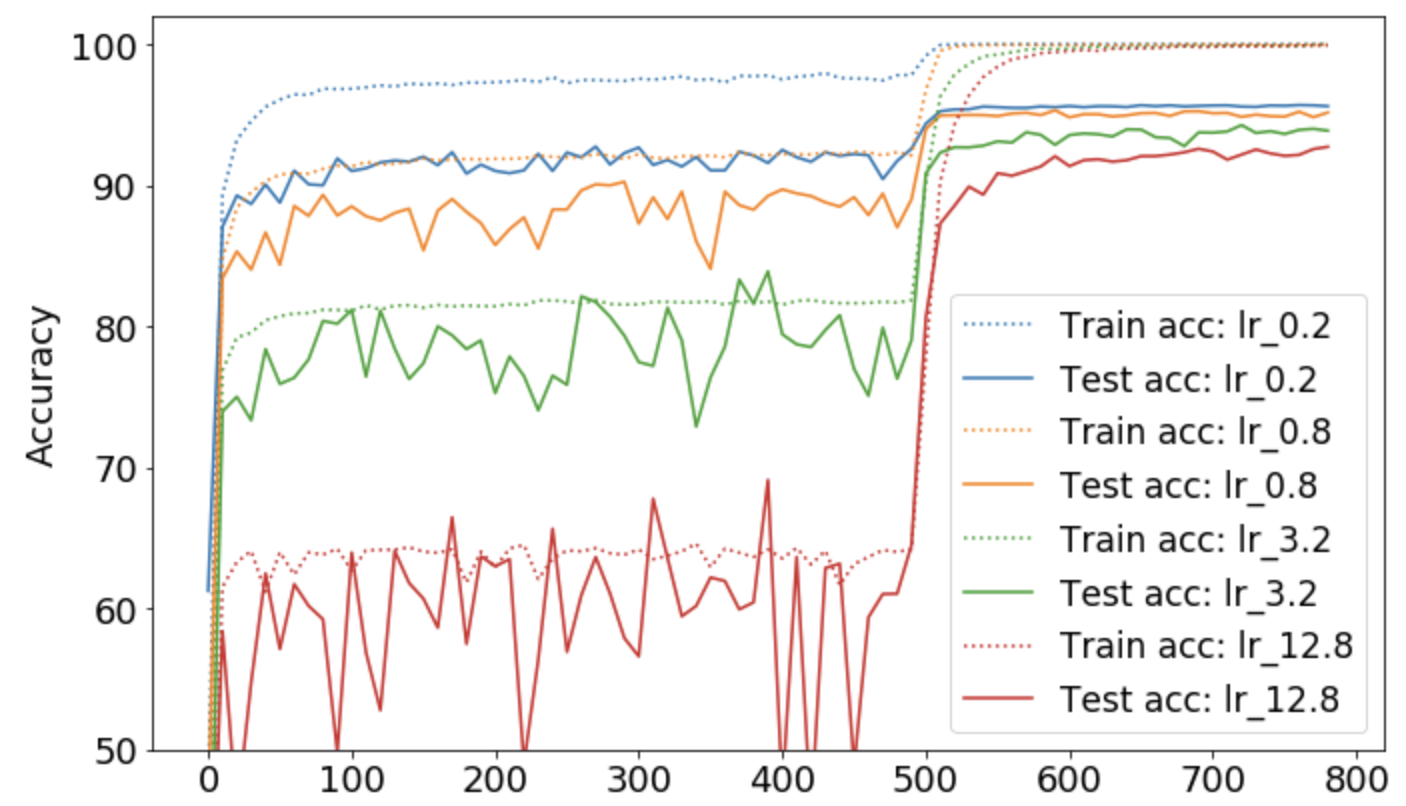}
 		\caption{Train/Test Accuracy}
 		\label{fig:which_equilibrium}
 	\end{subfigure}%
 	\caption{\small \textbf{(a)-(c)}  Achieving SOTA test accuracy by $0.9$-momentum SGD with small learning rates (the blue line). The initial learning rate is 0.1, initial WD factor is 0.0005. The label \texttt{wd\_x\_y\_lr\_z\_u} means dividing WD factor by 10 at epoch $x$ and $y$, and dividing LR by 10 at epoch $z$ and $u$. For example, the blue line means dividing LR by 10 twice at epoch 0, i.e. using an initial LR of 0.001 and dividing LR by 10 at epoch 5000. The red line is baseline.  \textbf{(d)} Equilibrium of smaller intrinsic LR leads to higher test accuracy on CIFAR after LR decay by 10. PreResNet32 trained with SGD without momentum and with WD factor 0.0005. }
 	\label{fig:small_lr_sota}
 \end{figure}
 
Since each phase in Step decay except the last one is allowed to reach equilibrium, the above conjecture suggests the generalization error of Step Decay schedule only depends on the intrinsic LR for its last equilibrium, namely the second-to-last phase. Thus, Step Decay could be abstracted into the following general \emph{two-phase training} paradigm, where the only hyper-parameter of SGD that affects generalization is the intrinsic LR, $\lambda_e$:
\begin{enumerate}[topsep=0pt,itemsep=0pt]
\item \textbf{SDE Phase.} Reach the equilibrium of (momentum) SGD with $\lambda_e$ (as fast as possible).
\item \textbf{Gradient Flow Phase.} Decay the learning rate by a large constant, e.g., 10, and set $\lambda_e=0$. Train until loss is zero.
\end{enumerate}
The above training paradigm  says for good generalization, what we need is only \emph{reaching the equilibrium of small (intrinsic) LR and then decay LR and stop quickly}. In other words, the initial large LR should not be necessary to achieve high test accuracy. Indeed our experiments show that networks trained directly with the small intrinsic LR, though necessarily for much longer due to slower mixing, also achieve the same performance. See \Cref{fig:small_lr_sota} for SGD with momentum and \Cref{fig:small_lr_sgd_sota} for vanilla SGD.
 
\myparagraph{So what's the benefit of early large learning rates?} Empirically we observed that initial large (intrinsic) LR leads to a faster convergence of the training process to the equilibrium. See the red lines in \Cref{fig:small_lr_sota}. A natural guess for the reason is that directly reaching the equilibrium of small intrinsic LR from the initial distribution is slower than to first reaching the equilibrium of a larger intrinsic LR and then the equilibrium of the target small intrinsic LR. This has been made rigorous for the mixing time of SDE in parameter space~\citep{shi2020learning}. In our setting, we show in \Cref{appsec:benefit_early_large_lr} that this argument makes sense at least for norm convergence: the initial large LR reduces the gap between the current norm and the stationary value corresponding to the small LR, in a much shorter time. In \Cref{fig:unfavorable_init}, we show that the early large learning rate is crucial for the learnability of normalized networks with initial distributions with extreme magnitude. Intriguingly, though without a theoretical analysis, early large learning rate experimentally (see  \Cref{fig:small_lr_sota_norm}) accelerates norm convergence and convergence to equilibrium even with momentum.

\myparagraph{But is it worth waiting for the equilibrium of small (intrinsic) LR?} In \Cref{fig:which_equilibrium} we show that different equilibrium does lead to different performance after the final LR decay. Given this experimental result we speculate the basins of different scales in the optimization landscape seems to be nested, i.e., a larger basin can contain multiple smaller basins of different performances. And reaching the equilibrium of a smaller intrinsic LR seems to be a stronger regularization method, though it also costs much more time.

\myparagraph{Batch size and linear scaling rule:} Recall the batch loss is $\Loss(\vw; \Batch) = \frac{1}{\lvert \Batch \rvert}\sum_{b \in \Batch} \ell_\Batch(\vw; \vx_b, y_b)$. If $\ell_\Batch$ is independent of $\Batch$, such as GroupNorm or LayerNorm is used instead of BatchNorm, we have $\Sigma^{B}_\vw = \frac{1}{B}\Sigma^{1}_\vw$, where $\Sigma^{B}_\vw$ is the noise covariance when the batch size is $B$. Therefore, let $\vW^{B,\eta}_t$  denote the solution in \Cref{eq:sgd-sde}, we have $\vW^{B,\eta}_t = \vW^{1,\frac{\eta}{B}}_{Bt}$, given that the initialization are the same, i.e. $\vW^{B,\eta}_0 = \vW^{1,\frac{\eta}{B}}_0$. In other words, up to a time rescaling, doubling the batch size $B$ is equivalent to halving down LR $\eta$ for all losses in the SDE regime, which is a.k.a.~\emph{linear scaling rule}~\citep{goyal2017accurate}. Therefore, in the special case of sclale-invariant loss and WD, we conclude that $\frac{\lambda_e}{B}$ alone  determines the equilibrium of SDE. 

However, this analysis has the following weaknesses and thus we treat batch size $B$ as a fixed hyper-parameter in this paper: (1) It's less general. For example, it doesn't work for BN, especially when batch size goes to 0, as $\Sigma^{B}_\vw $ can be significantly different from $\frac{1}{B}\Sigma^{1}_\vw$ due to the noise in batch statistics; (2) Unlike the equivalence between LR and WD, which holds exactly even for SGD \citep{li19learningrate}, the equivalence between batch size $B$ and intrinsic LR $\lambda_e$, only holds in the regime where $\lambda_e$ and $B$ are small and is less meaningful for the more interesting direction, i.e., acceleration via large-batch training~\citep{smith2020generalization}.

\section{Experimental Evidence of Theory} \label{sec:evidence}
\subsection{Equilibrium is independent of the initial distribution}\label{subsec:eq_idp_init}
In this subsection we aim to show that the equilibrium only depends on the intrinsic LR, $\lambda_e = \eta\lambda$, and is independent of the initial distribution of the weights and individual values of $\eta$ and $\lambda$. 

\begin{figure}[t]
	\begin{minipage}{\textwidth}
		\centering
		\begin{subfigure}[b]{0.4\textwidth}
			\centering
			\includegraphics[scale=0.4]{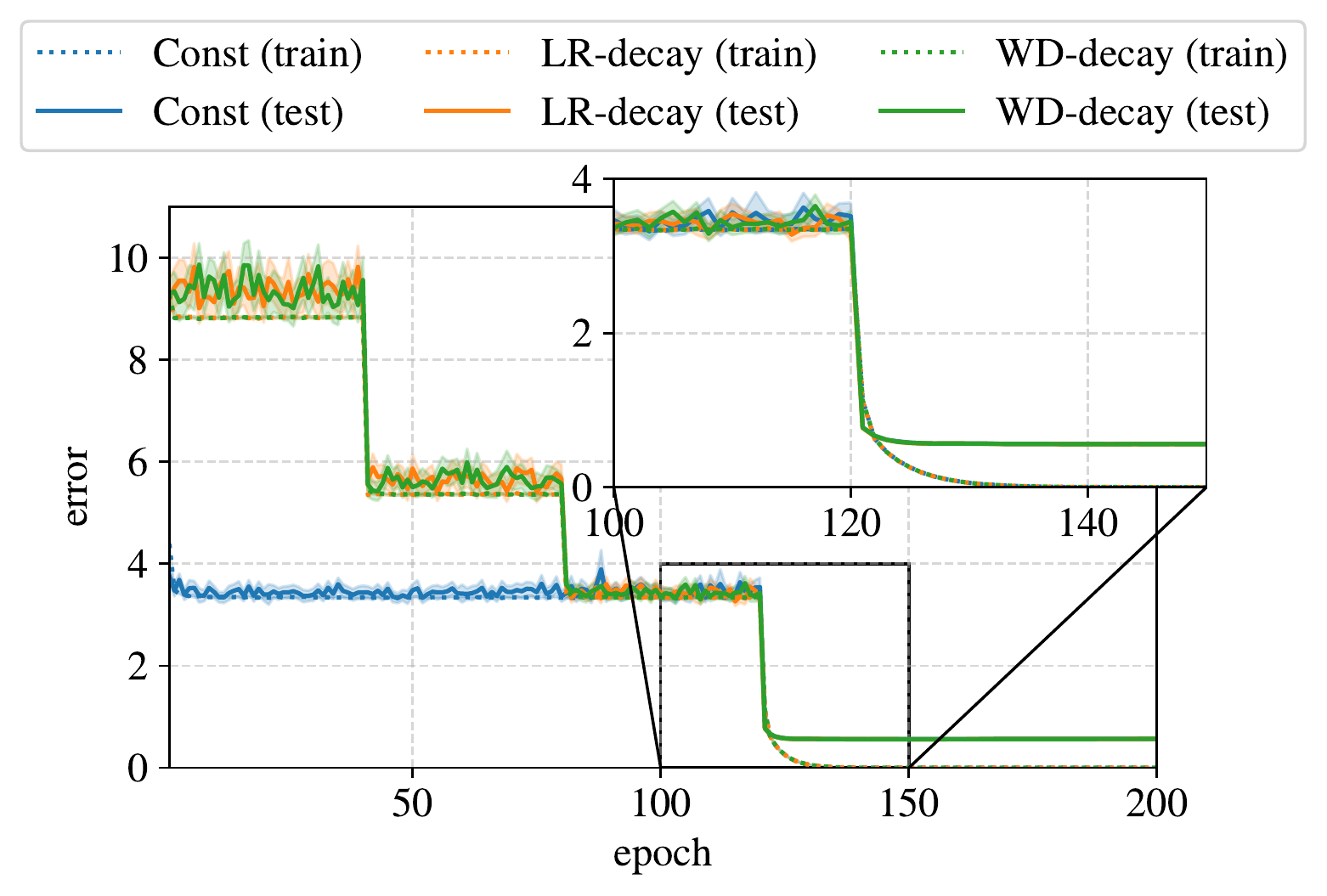}
			\caption{Train/test accuracy} \label{fig:mnist-acc}
		\end{subfigure}%
		\begin{subfigure}[b]{0.6\textwidth}
			\centering
			\includegraphics[scale=0.4]{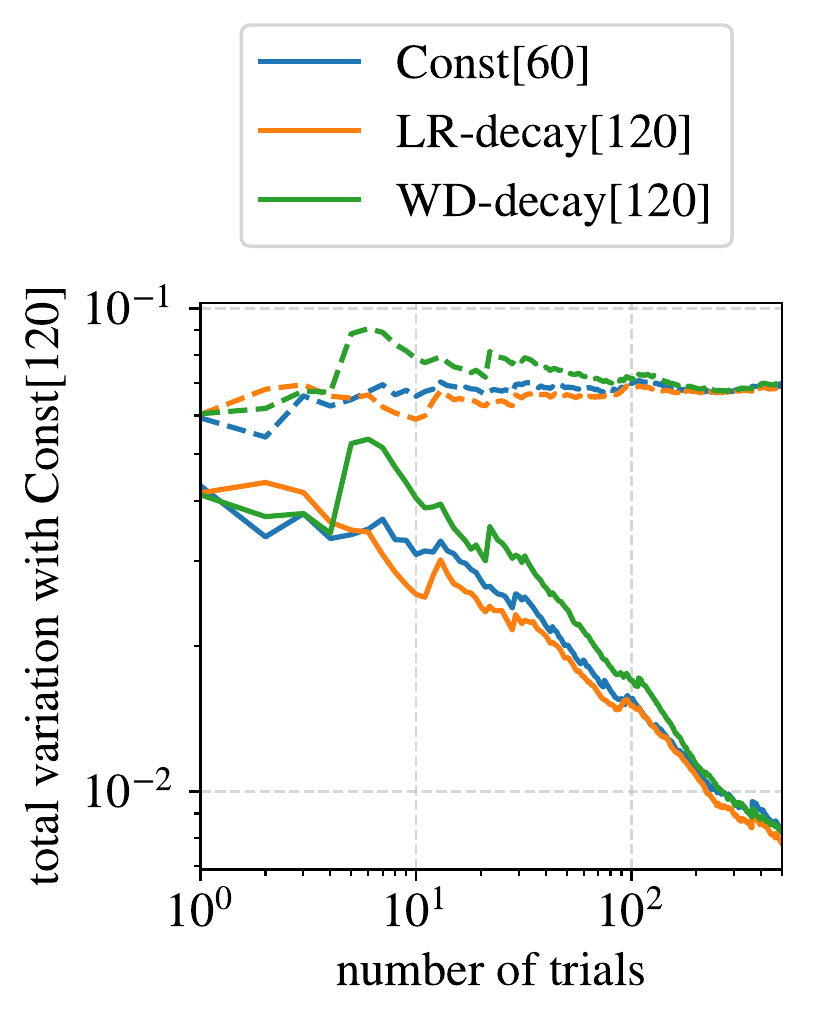}
			\includegraphics[scale=0.4]{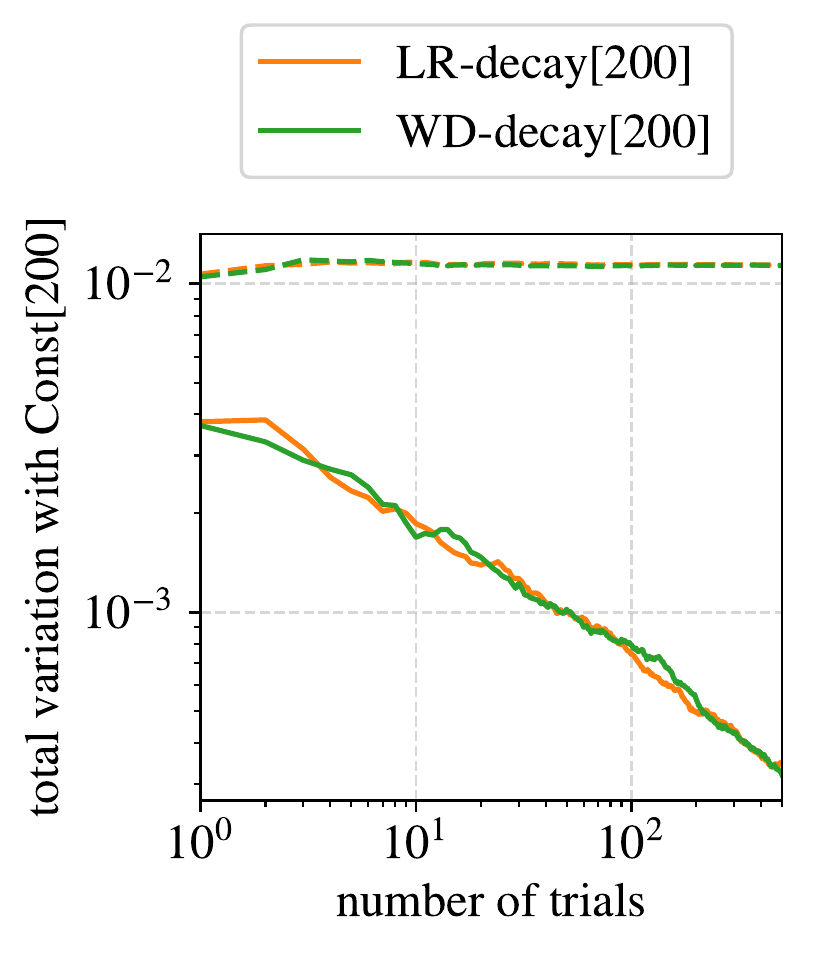}
			\caption{Total variation} \label{fig:mnist-tv}
		\end{subfigure}%
	\end{minipage}

	\caption{\footnotesize A simple 4-layer CNN trained on MNIST with three schedules converge to the same equilibrium after intrinsic LRs become equal at epoch $81$. \textbf{(a)} The train/test errors (averaged over 500 trials) are almostly the same from epoch $81$. \textbf{(b)} We estimate the total variation between the empirical distribution of the predictions on test images for neural nets trained with schedule \texttt{Const} and other schedules for 120/200 epochs (solid lines). The estimated value decreases with the number of trials. For comparison, the dashed lines are the sum of averaged test errors of each pair of training processes, which can be seen as baselines since the sum is the total variation when the set of images that lead to wrong predictions for the two training processes are completely different.} \label{fig:mnist}
\end{figure}

\myparagraph{MNIST Experiments.} We use a simple 4-layer CNN for MNIST. To highlight the effect of scale-invariance, we make the CNN scale-invariant by fixing the last linear layer as well as the affine parameters in every BN. \Figref{fig:mnist-acc} shows the train/test errors for three different schedules, \texttt{Const}, \texttt{LR-decay} and \texttt{WD-decay}. Each error curve is averaged over 500 independent runs, where we call each run as a \textit{trial}. \texttt{Const} initiates the training with $\eta = 0.1$ and $\lambda = 0.1$. \texttt{LR-decay} initiates the training with 4 times larger LR and decreases LR by a factor of $2$ every $40$ epochs. \texttt{WD-decay} initiates the training with 4 times larger WD and decreases WD by a factor of $2$ every $40$ epochs. All these three schedules share the same intrinsic LR from epoch $81$ to $120$, and thus reach the same train/test errors in this phase as we have conjectured. Moreover, after we setting $\eta = 0.01$ and $\lambda = 0$ at epoch $121$ for fine-tuning, all the schedules show the same curve of decreasing train/test errors, which verifies the strong form of our conjecture.

\Figref{fig:mnist-tv} measures the total variation between predictions of neural nets trained for 120 and 200 epochs with different schedules. Given a pair of distributions $\mathcal{W}, \mathcal{W}'$ of neural net parameters (e.g., the distributions of neural net parameters after training with \texttt{LR-decay} and \texttt{WD-decay} for $200$ epochs), we enumerate each input image $x$ from the test set and compute the total variation between the empirical distributions $\mathcal{D}_x, \mathcal{D}_x'$ of the class prediction on $x$ for weights sampled from $\mathcal{W}, \mathcal{W}'$, where $\mathcal{D}_x, \mathcal{D}_x'$ are estimated through the trials. \Figref{fig:mnist-tv} shows that the average total variation over all test inputs decrease with the number of trials, which again suggests that the mixing happens in the function space.

\myparagraph{CIFAR-10 Experiments.} We use PreResNet32 for CIFAR10 with data augmentation and the batch size is 128. We modify the downsampling part according to the Appendix C in \citep{li19learningrate} and fix the last layer and 
$\gamma$, $\beta$ in every BN, to ensure the scale invariance.  In \Cref{fig:random_hist} we focus on the comparison between the performance of the networks within and after leaving the equilibrium, where the networks are initialized differently via different LR/WD schedules before switching to the same intrinsic LR. We repeat this experiment with VGG16 on CIFAR-10 (\Cref{fig:random_hist_vgg16_cifar10}) and PreResNet32 on CIFAR-100 (\Cref{fig:random_hist_cifar100}) in appendix and get the same results. A direct comparison between the effect of LR and WD can be found in \Cref{fig:7decay}.
\begin{figure}
        \centering
        \begin{subfigure}[b]{0.41\textwidth}
                \centering
				\includegraphics[width=\linewidth]{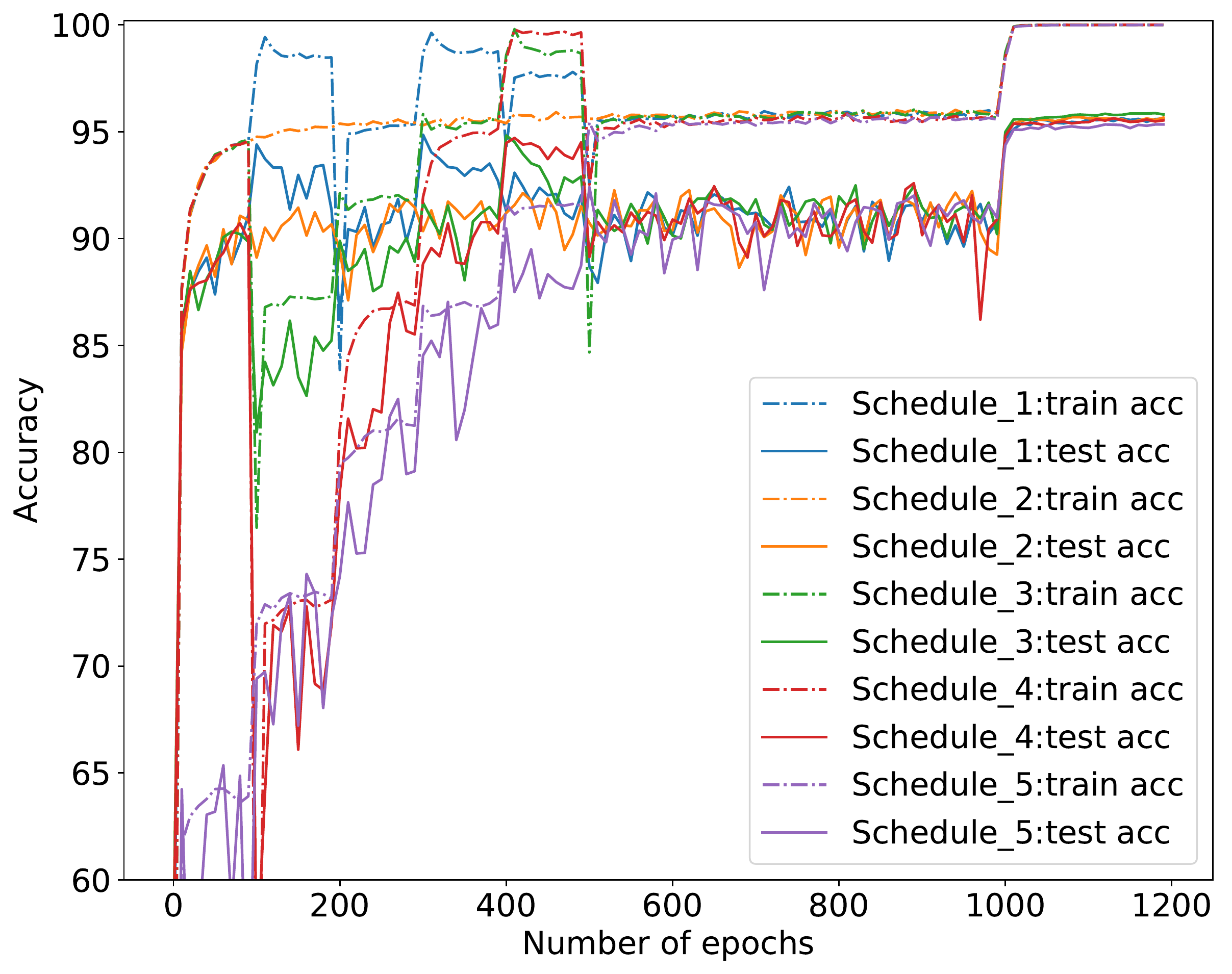}
				\caption{Train/test accuracy}
        \end{subfigure}%
        \begin{subfigure}[b]{0.58\textwidth}
                \centering
                \includegraphics[width=\linewidth]{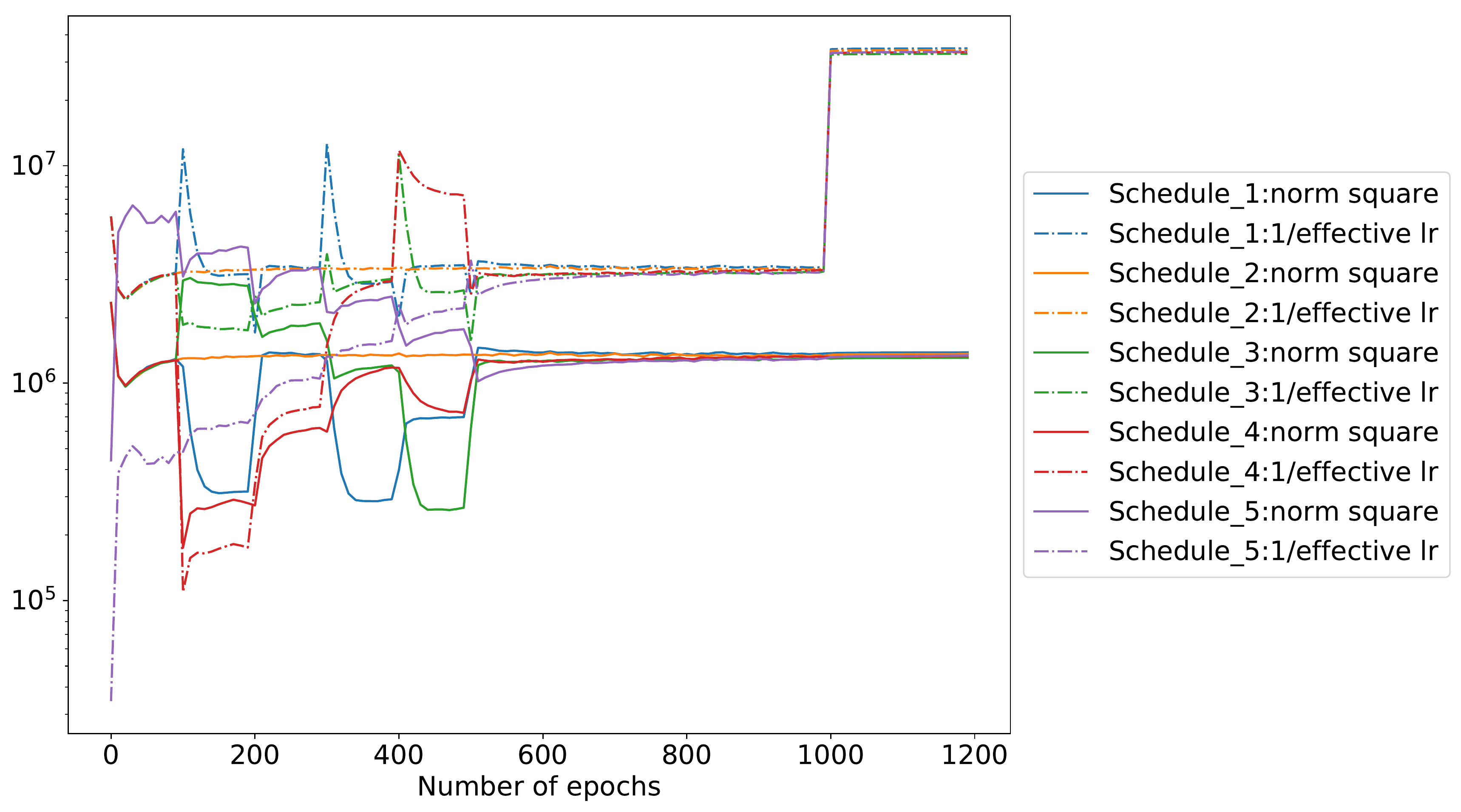}
                \caption{Norm and effective LR}
        \end{subfigure}%
        \caption{\small PreResNet32 trained by SGD with 5 random LR/WD schedules in the first 500 epochs converge to the same equilibrium when LR and WD factor are set the same at epoch 500 -- These different trajectories exhibit similar test/train accuracy,  norm and effective LR. Moreover, they achieve the same best test accuracy ($\sim 95\%$, the same as that with momentum) after decaying LR and removing WD at epoch 1000, suggesting that the equilibrium is independent of initialization. See details of the schedules in  \Cref{table:random_history}. (Appendix)}
        \label{fig:random_hist}
\end{figure}

\subsection{Reaching Equilibrium only takes $O(1/(\lambda\eta))$ steps}\label{subsec:fast_eq_exp}

\begin{figure}[t]
        \centering
        \begin{subfigure}[b]{0.41\textwidth}
                \centering				\includegraphics[width=\linewidth]{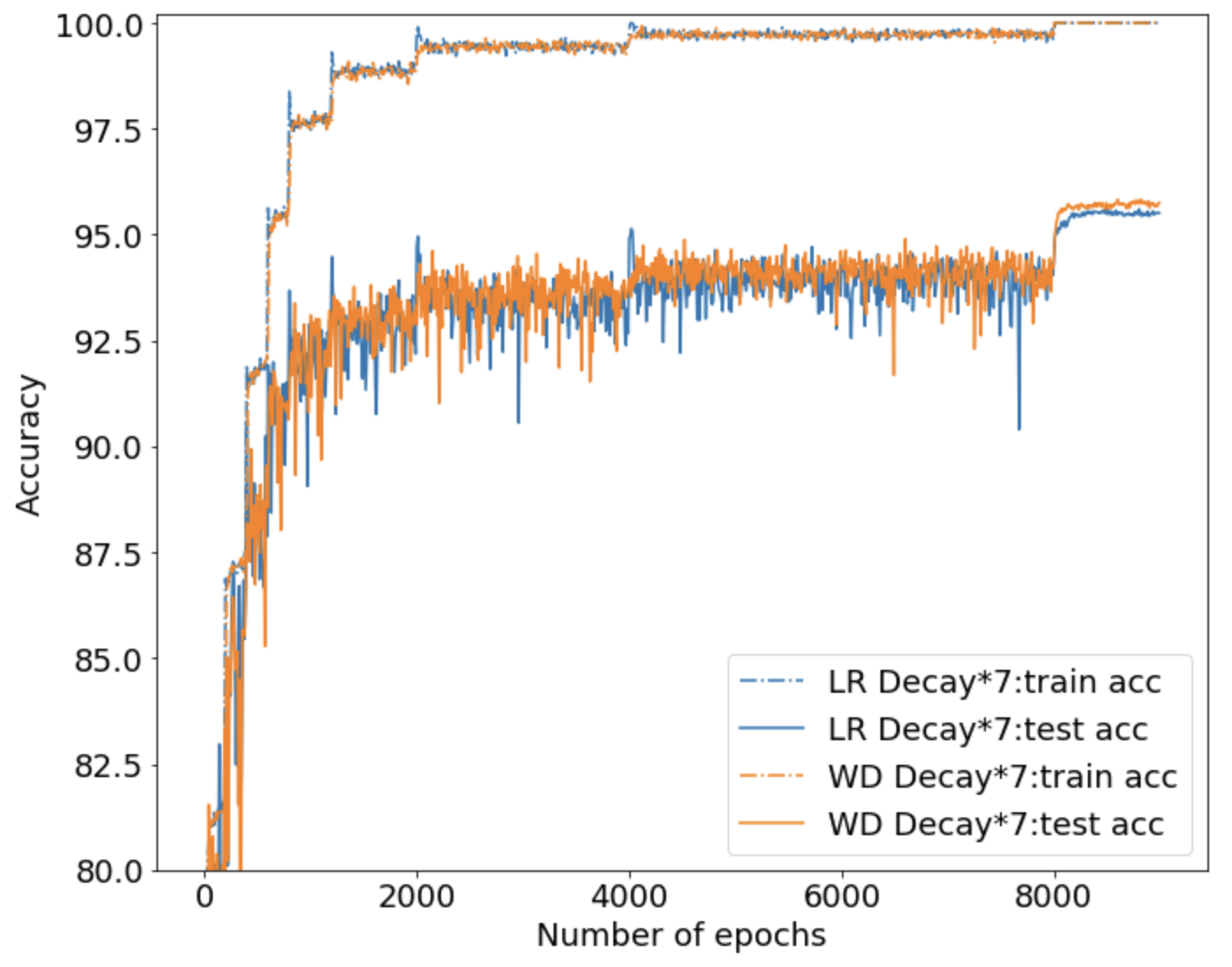}
				\caption{Train/test accuracy}
        \end{subfigure}%
        \begin{subfigure}[b]{0.58\textwidth}
                \centering
                \includegraphics[width=\linewidth]{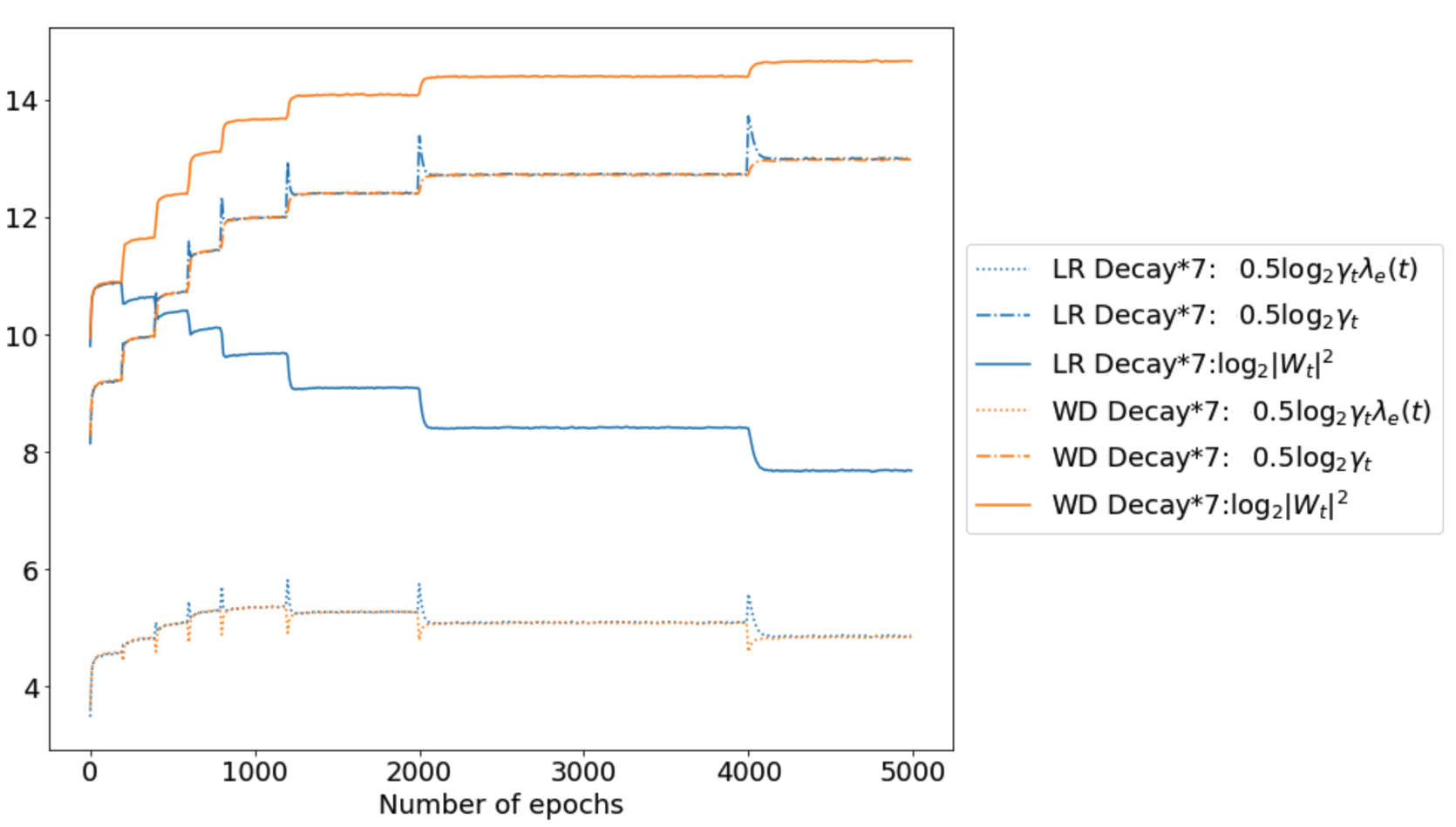}
                \caption{Norm and 1/ effective LR, $\gamma_t^{0.5}$}
        \end{subfigure}%
        \caption{\small Smaller intrinsic LR takes longer time to stabilize its norm. We train two PreResNet32 by SGD with the same initial LR$,\eta=3.2$ and WD factor, $\lambda =0.0005$ without momentum. The LR/ WD factor are divided by 2 at epoch 200, 400, 600, 800, 1200, 2000 and 4000 respectively. Still, the networks share almost the same effective LR and train/test accuracy for most of the time. The best test accuracies for both are achieved by removing WD and dividing LR by 10 at epoch 8000.}
        \label{fig:7decay}
\end{figure}

\begin{figure}[!htbp]
        \centering
        \begin{subfigure}[b]{\textwidth}
                \centering
				\includegraphics[width=\linewidth]{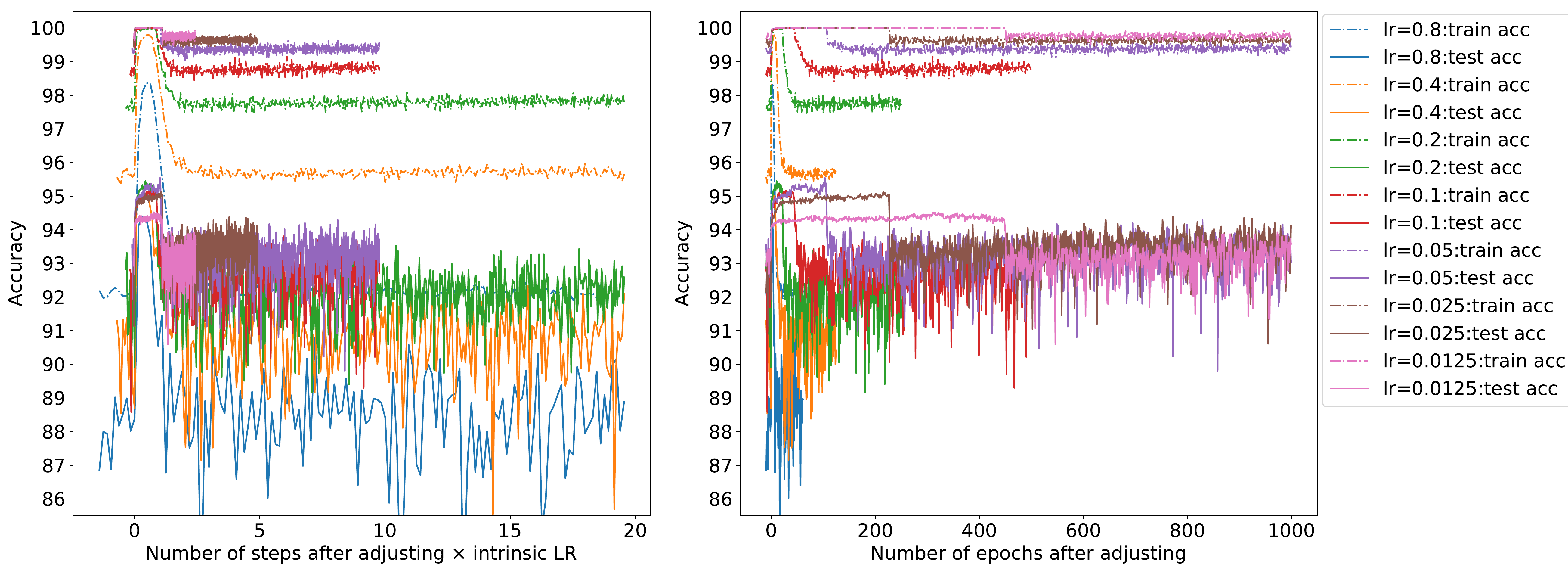}
				\caption{Train/test accuracy}
        \end{subfigure}%
        
        \begin{subfigure}[b]{\textwidth}
                \centering
                \includegraphics[width=\linewidth]{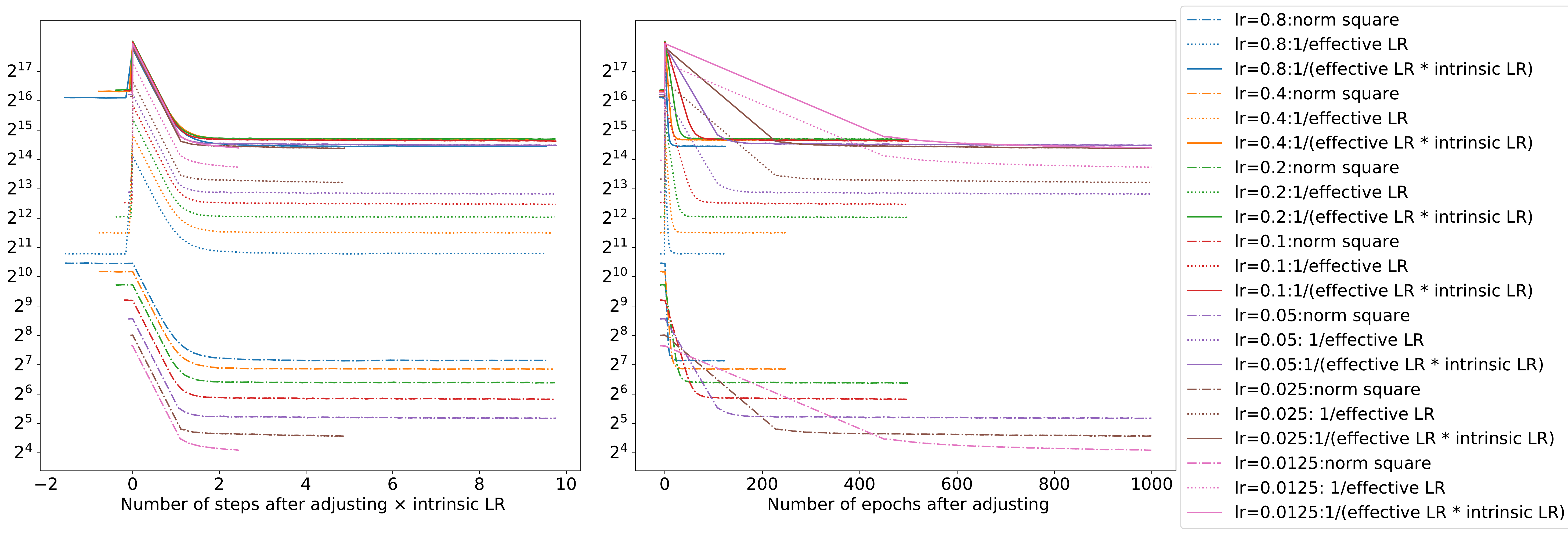}
                \caption{Norm and effective LR}
        \end{subfigure}%
        \caption{ VGG16 was trained on CIFAR10 with BN + SGD and different intrinsic LRs. Then LR and WD were changed while maintaining their product (i.e., intrinsic LR).  Number of steps to  reach equilibrium again was measured. It scales inversely with intrinsic LR, supporting Fast Equilibrium Conjecture.  }
        \label{fig:vgg_fast_conjecture}
\end{figure}

In \Cref{fig:7decay} we show that convergence of norm is a good measurement for reaching equilibrium, and it takes longer time for smaller intrinsic LR $\lambda_e$. The two networks are trained with the same sequence of intrinsic learning rates, where the first schedule (blue) decays LR by 2 at epoch , and the second schedule decays WD factor by 2 at the same epoch list. Note that the effective LR almost has the same trend as the training accuracy. Since in each phase, the effective LR $\gamma_t^{-0.5} \propto \norm{\vW_t}^{-2}$, we conclude that the convergence of norm suggests SGD reaches the equilibrium. 

In \Cref{fig:vgg_fast_conjecture} we provide  experimental evidence that the mixing time to equilibrium in function space scales to $\frac{1}{\eta\lambda}$. Note in \Cref{eqn:norm_mixing2}, the convergence of norm also depends on the initial value. Thus in order to reduce the effect of initialization on the time towards equilibrium, we use the setting of Figure 3 in \citep{li19learningrate}, where we first let the networks with the same architecture reach the equilibrium of different intrinsic LRs, and we decay the LR by $10$ and multiplying the WD factor by $10$ simultaneously. In this way the intrinsic LR is not changed and the equilibrium is still the same. However, the effective LR is perturbed far away from the equilibrium, i.e. multiplied by $0.1$. And we measure how long does it takes SGD to recover the network back to the equilibrium and we find it to be almost linear in $1/\lambda\eta$.

\section{Conclusion and Open Questions}
We pointed that use of normalization in today's state-of-art architectures  today leads to a mismatch with traditional mathematical views of optimization.  To bridge  this gap we develop the mathematics of  SGD + BN + WD in scale-invariant nets, in the process identifying a new hyper-parameter ``intrinsic learning rate'', $\lambda_e =\eta\lambda$, for which appears to determine  trajectory evolution and network performance  after reaching equilibrium. Experiments suggest time to equilibrium in {\em function space} is only $O(\frac{1}{\lambda_e})$, dramatically lower than the usual  exponential upper bound for mixing time in parameter space. Our  \emph{fast equilibrium conjecture} about this  may guide future theory.  The conjecture suggests a more general two-phase training paradigm, which could be potentially interesting to practitioners and lead to better training.

Our theory shows that convergence of norm is a good sign for having reached equilibrium. However, we still lack a satisfying  measure of the progress of training, since empirical risk is not good. Finally, it would be good to \emph{understand why reaching equilibrium helps regularization}.

\section*{Acknowledgement}

ZL and SA acknowledge support from NSF, ONR, Simons Foundation, Schmidt Foundation, Mozilla
Research, Amazon Research, DARPA and SRC.

\bibliography{main}

\newpage

\appendix

\section{Batch Normalization} \label{app:bn}

Batch normalization (BN)~\citep{ioffe2015batch} is one of the most commonly-used normalization schemes. Given a mini-batch of inputs $\{z_b\}_{b \in \Batch}$ from the last layer, a batch normalization layer first normalizes the inputs by subtracting the mean $\mu_{\Batch} := \frac{1}{\lvert \Batch \rvert} \sum_{b \in \Batch} z_b$ and dividing the variance $\sigma^2_{\Batch} = \frac{1}{\lvert \Batch \rvert} \sum_{b \in \Batch} (z_b - \mu_{\Batch})^2$, and then applies a linear transformation with trainable parameters $\gamma$ and $\beta$:
\[
\BN(z_b) = \gamma \frac{z_b - \mu_{\Batch}}{\sigma_{\Batch}} + \beta.
\]
Typically BN is placed between the linear transformation and activation function. This makes the loss invariant to the re-scaling of weights in the linear transformation preceding the BN. If we fix the weights in the last linear layer as suggested by~\citep{hoffer2018fix} and put BN after every linear transformation, then the loss is invariant to all its parameters (see Appendix C of \citep{li2020exp} for more details).
\section{Missing derivation and proofs}

For scale-invariant loss function $\Loss(\vw; \Batch)$, we have the following lemma on the covariance matrix of gradient noise:
\begin{lemma}\label{lem:perp_cov}
	If $\Loss(\vw; \Batch)$ is scale-invariant with respect to $\vw$, then
	\begin{enumerate}
		\item $\mSigma_{c\vw} = c^{-2}\mSigma_{\vw}$ for any $c > 0$.
		\item $\normtwosm{\mSigma_{\vw}^{1/2} \vw} = \vw^\top \mSigma_{\vw} \vw = 0$.
	\end{enumerate}
\end{lemma}
\begin{proof}
	Note that the expectation of scale-invariant functions is scale-invariant, so $\Loss(\vw)$ is scale-invariant. The first bullet can be proved by combining the definition of $\mSigma_{c\vw}$ and
	\[
	\nabla \Loss(c\vw; \Batch) - \nabla \Loss(c\vw) = \frac{1}{c}\left(\nabla \Loss(\vw; \Batch) - \nabla \Loss(\vw)\right).
	\]
	For the second bullet, by homogeneity we have $\dotp{\nabla \Loss(\vw; \Batch)}{\vw} = 0$ and $\dotp{\nabla \Loss(\vw)}{\vw} = 0$, so $\dotp{\nabla \Loss(\vw; \Batch) - \nabla \Loss(\vw)}{\vw} = 0$. This implies $\vw^\top \mSigma_{\vw} \vw = \E[\dotp{\nabla \Loss(\vw; \Batch) - \nabla \Loss(\vw)}{\vw}^2] = 0$.
\end{proof}

To prove \Cref{lm:evolution}, we will need to use the It\^{o}'s lemma, which is stated below:

\begin{lemma}[It\^{o}'s Lemma]\label{lem:ito}
	Suppose $\mX_t = (X_t^1, X_t^2, \ldots, X_t^d)$ is a vector of It\^{o}'s processes s.t. 
	\[ \dd \mX_t = \vmu_t \dd t + \mG_t \dd \mB_t, \]

we have for any twice differentiable function $f$,
\[
{\displaystyle {\begin{aligned}df(t,\mathbf {X} _{t})&={\frac {\partial f}{\partial t}}\,dt+\left(\nabla _{\mathbf {X} }f\right)^{T}\,d\mathbf {X} _{t}+{\frac {1}{2}}\left(d\mathbf {X} _{t}\right)^{T}\left(H_{\mathbf {X} }f\right)\,d\mathbf {X} _{t},\\&=\left\{{\frac {\partial f}{\partial t}}+\left(\nabla _{\mathbf {X} }f\right)^{T}{\boldsymbol {\mu }}_{t}+{\frac {1}{2}}\operatorname {Tr} \left[\mathbf {G} _{t}^{T}\left(\nabla^2_{\mathbf {X} }f\right)\mathbf {G} _{t}\right]\right\}dt+\left(\nabla _{\mathbf {X} }f\right)^{T}\mathbf {G} _{t}\,d\mathbf {B} _{t}\end{aligned}}}\]
\end{lemma}

Recall the following original SDE in the space of $\mW$. Below we will prove \Cref{lm:evolution} by It\^{o}'s Lemma.
\begin{align*} 
\dd \vW_{t} &= - \eta \left(\nabla \Loss(\vW_t) \dd t+ (\mSigma_{\vW_t})^{\frac{1}{2}} \dd \vB_t\right) - \lambda_e \vW_t \dd t. 
\end{align*}

\mainlemma*
\begin{proof}[Proof for \Cref{lm:evolution}]
    We can prove \eqref{eq:sde-hW} and \eqref{eq:sde-rho} by It\^{o}'s Lemma. For \eqref{eq:sde-rho}, note that $G_t = \norm{\mW_t}^2$, we have 
	\begin{align*}
		\dd G_t &= 2 \mW_t^\top \dd \mW_t + \dd \mW_t^\top \dd \mW_t\\
		&= 2 \dotp{\vW_t}{- \eta \left(\nabla \Loss(\vW_t) \dd t+ (\mSigma_{\vW_t})^{\frac{1}{2}} \dd \vB_t\right) - \lambda_e \vW_t \dd t} + \eta^2 \Tr(\mSigma_{\vW_t}) \dd t.
	\end{align*}

	By scale-invariance and \Cref{lem:perp_cov}, $\dotp{\vW_t}{\nabla \Loss(\vW_t)} = 0$, $\dotp{\vW_t}{(\mSigma_{\vW_t})^{\frac{1}{2}} \dd \vB_t} = 0$, $\Tr(\mSigma_{\vW_t}) = \frac{1}{G_t}\Tr(\mSigma_{\hW_t})$. So we can simplify the formula to conclude that
	\[
		\dd G_t = -2\lambda_e G_t \dd t + \frac{\eta^2}{G_t} \Tr(\mSigma_{\hW_t}) \dd t.
	\]
	
	For \eqref{eq:sde-hW}, let $\vv\in\R^d$ be an arbitrary vector, then
	\begin{align*}
	\dd\dotp{\vv}{\hW_t} =& \dotp{\frac{\partial \dotp{\vv}{\hW_t}}{\partial \mW_{t}}}{ \dd \mW_{t}} + \frac{1}{2}  (\dd \mW_{t})^\top\frac{\partial^2 \dotp{\vv}{\hW_t}}{\left(\partial \mW_{t}\right)^2 }\dd \mW_{t}\\
	= & \dotp{\frac{1}{\norm{\vW_t}} \left(\vv - \dotp{\vv}{\hW_t} \hW_{t}\right)}{ \dd \mW_{t}} 
	+ 	 \frac{\eta^2}{2}  \Tr\left( (\mSigma_{\mW_t})^{\frac{1}{2}} \frac{\partial^2 \dotp{\vv}{\hW_t}}{\left(\partial \mW_{t}\right)^2 }(\mSigma_{\mW_t})^{\frac{1}{2}}\right) \dd t
	\end{align*}
	By scale-invariance and \Cref{lem:perp_cov}, $\dotp{\vW_t}{\nabla \Loss(\vW_t)} = 0$, $(\mSigma_{\vW_t})^{\frac{1}{2}} \vW_t  = \bm{0}$, which means the column span of $\mSigma_{\vW_t}$ is orthogonal to $\vW_t$.
	Thus we can apply \Cref{lm:invariant-hess} below and get 
	 \[ \Tr\left( (\mSigma_{\mW_t})^{\frac{1}{2}} \frac{\partial^2 \dotp{\vv}{\hW_t}}{\left(\partial \mW_{t}\right)^2 }(\mSigma_{\mW_t})^{\frac{1}{2}}\right) = -\frac{\vv^\top \hW_t}{\norm{\vW_t}^2}\Tr\left( \mSigma_{\mW_t}\right).\] 
	 
	 Then we have
	\[
		\dd\dotp{\vv}{\hW_t} = \dotp{\frac{\vv}{\norm{\vW_t}}}{- \eta \left(\nabla \Loss(\vW_t) \dd t+ (\mSigma_{\vW_t})^{\frac{1}{2}} \dd \vB_t\right)} - \frac{\eta^2}{2 \norm{\vW_t}^2} \dotp{\hW_t}{\vv} \Tr(\mSigma_{\vW_t}) \dd t.
	\]
	By scale-invariance, this can be simplified to the following formula:
	\[
		\dd\dotp{\vv}{\hW_t} = -\frac{\eta}{\norm{\vW_t}^2 }\dotp{\vv}{\nabla \Loss(\hW_t) \dd t+ (\mSigma_{\hW_t})^{\frac{1}{2}} \dd \vB_t} - \frac{\eta^2}{2 \norm{\vW_t}^4} \Tr(\mSigma_{\hW_t}) \dotp{\vv}{\hW_t}  \dd t,
	\]
	which proves \eqref{eq:sde-hW}, since $\vv$ is arbitrary.
\end{proof}

\begin{lemma} \label{lm:invariant-hess}
For any twice differentiable function $g:\R^d\to \R$, the Hessian matrix of $f(\vw) := g(\frac{\vw}{\norm{\vw}})$ satisfies that, $\forall \vv^\top \vw = 0$,
\begin{equation}
	\vv^\top \nabla^2 f(\vw) \vv =\frac{1}{\norm{\vw}^2}\left(\vv^\top \nabla^2 g(\hvw)\vv- \hvw^\top \nabla g(\hvw) \norm{\vv}^2\right).
\end{equation}
where $\hvw = \frac{\vw}{\norm{\vw}}$.
\end{lemma}
\begin{proof}
	For any $\vv\in\R^d$ s.t. $\vv^\top \vw=0$, we define $h(\lambda) = f(\vw+\lambda\norm{\vw}\vv)$. Then by definition, $\norm{\vw}^2\vv^\top \nabla^2 f(\vw) \vv = h''(0)$.
	
	On the other hand, note $\vv^\top \vw = 0$, we have 
	
	\begin{align*}
	\frac{\vw+\lambda\norm{\vw}\vv}{\norm{\vw+\lambda\norm{\vw}\vv}} &= \frac{\vw+\lambda\norm{\vw}\vv}{\norm{\vw}} \frac{1}{\sqrt{1+\lambda^2\norm{\vv}^2}} \\
	&= (\hvw+ \lambda\vv)(1-\frac{\lambda^2\norm{\vv}^2}{2} +O(\lambda^4)) \\
	&= \hvw +\lambda\vv-\frac{\lambda^2\norm{\vv}^2}{2}\hvw + O(\lambda^3).
	\end{align*}
	
	Thus 
	\begin{align*}
	h(\lambda) &= g(\hvw +\lambda\vv-\frac{\lambda^2\norm{\vv}^2}{2}\hvw + O(\lambda^3) ) \\
	&= g(\hvw)+ \lambda \nabla  g(\hvw)^\top \vv + \frac{\lambda^2}{2} \left( \vv^\top\nabla^2g(\hvw)\vv - \norm{\vv}^2 \nabla g(\hvw)^\top \vv \right) + O(\lambda^3),
	\end{align*}
	from which we conclude $h''(0) =\vv^\top\nabla^2g(\hvw)\vv - \norm{\vv}^2 \nabla g(\hvw)^\top \vv $.
\end{proof}

\begin{proof}[Proof for \Cref{lm:gamma-decay}]
	By~\eqref{eq:gamma-avg}, we can upper bound and lower bound $\gamma_t$ by
	\begin{align*}
		\gamma_t &\le e^{-4 \lambda_et} \gamma_0 + 2 \int_{0}^{t} e^{-4 \lambda_e (t-\tau)} (1+\epsilon)\sigma^2 d\tau \le  e^{-4 \lambda_et} \gamma_0 + \frac{1}{2\lambda_e}(1 - e^{-4 \lambda_e t}) (1+\epsilon)\sigma^2 d\tau. \\
		\gamma_t &\ge e^{-4 \lambda_et} \gamma_0 + 2 \int_{0}^{t} e^{-4 \lambda_e (t-\tau)} \sigma^2 d\tau \ge e^{-4 \lambda_et} \gamma_0 + \frac{1}{2\lambda_e}(1 - e^{-4 \lambda_e t})\sigma^2 d\tau.
	\end{align*}
	Therefore, we have $\gamma_t = e^{-4 \lambda_et} \gamma_0 + \left(1 + O(\epsilon)\right) \frac{\sigma^2}{2\lambda_e}\left( 1 - e^{-4\lambda_e t}\right)$.
\end{proof}

\paragraph{Connection to Exp LR schedule:} ~\citep{li2020exp} shows that 
\[\vw_{t+1} \gets \vw_t - \eta(1-\lambda_e)^{-2t} \left(\nabla \Loss(\vw_t) + \vxi_t \right)\]
yields the same trajectory in function space as \Cref{eq:sgd-wd-noisy} for scale invariant loss $\Loss$. In fact, they also correspond to the same surrogate SDE \Cref{eq:gamma_W,eq:gamma}, where the exponent in the rate schedule is the intrinsic LR. 
\begin{lemma}
The following SDE with exponential LR is equivalent to \Cref{eq:gamma_W,eq:gamma}, where $\gamma_t = \frac{\norm{\mW_t}^4 e^{-4\lambda_e t}}{\eta^2}$.
\[\dd \vW_{t} = - e^{2\lambda_e t}\eta \left(\nabla \Loss(\vW_t) \dd t+ (\mSigma_{\vW_t})^{\frac{1}{2}} \dd \vB_t\right).\]
\end{lemma}

\begin{proof}
	By It\^{o}'s Lemma, let $\mU_t = e^{-\lambda_e t}\mW_t$, we have 
	\begin{align*}
	\dd \mU_{t} =& -\lambda_e \mU_t \dd t + e^{-\lambda_e t}\dd \mW_t \\
	= &-\lambda_e \mU_t \dd t  + e^{-\lambda_e t}\left(  \nabla \Loss(\vW_t) \dd t+ (\mSigma_{\vW_t})^{\frac{1}{2}} \dd \vB_t \right)\\
	=& -\lambda_e \mU_t \dd t  +  \nabla \Loss(\mU_t) \dd t+ (\mSigma_{\mU_t})^{\frac{1}{2}} \dd \vB_t,
	\end{align*}
	
where the last step is by scale-invariance.

Note $\mU_t$ has the same direction as $\mW_t$, i.e. $\overline{\mU_t} = \hW_t$, we can apply \Cref{lm:evolution} to get \Cref{eq:sde-hW,eq:sde-rho}, and thus get \Cref{eq:gamma_W,eq:gamma}, with $\gamma_t = \frac{\norm{\mU_t}^4}{\eta^2} = \frac{\norm{\mW_t}^4 e^{-4\lambda_e t}}{\eta^2}$.
\end{proof}

\section{Extension to Other Optimization Algorithms} \label{app:other}

\subsection{Momentum}
In this subsection we use momentum SGD as an example to show how does the discrete version of the fast equilibrium conjecture look like. Throughout this section we will assume all the momentum factors are constant, and we only care about the role of LR $\eta$ and WD factor $\lambda$ in the discrete dynamics.

For fixed LR $\eta$ and WD $\lambda$, the formula of SGD with momentum can be written as follows:
\begin{align*}
	\vv_{t+1} &\gets \beta\vv_t + (\nabla \Loss(\vw_t; \Batch_t) + \lambda\vw_t) \\
	\vw_{t+1} &\gets \vw_t - \eta\vv_{t+1},
\end{align*}
which is also equivalent to
\[\vw_{t+1} -\vw_{t} = \beta(\vw_{t}-\vw_{t-1}) -\eta (\nabla\Loss(\vw_t; \Batch_t) + \lambda\vw_t).\]

We can decouple the effect of WD from SGD by replacing $\eta\lambda$ by $\lambda_e$:
\[\vw_{t+1} -(1 +\beta -\lambda_e)\vw_{t} + \beta\vw_{t-1} =-\eta\nabla\Loss(\vw_t; \Batch_t).\]
By scale invariance of $\Loss$, letting $\vw'_t = \frac{\vw_t}{\sqrt{\eta}}$, we have 
\[\vw'_{t+1} -(1 +\beta -\lambda_e)\vw'_{t} + \beta\vw'_{t-1} =-\nabla\Loss(\vw'_t; \Batch_t),\]

which means the effect of $\eta$ in the new parametrization is no more than rescaling the initialization. This motivates as to define $\lambda_e = \eta \lambda$ as the effective WD, or intrinsic LR. 

Unlike vanilla SGD, the evolution of norm for momentum SGD is more complicated. However, a folklore intuition is that, if the gradient of loss $\Loss$ changes slowly, one can approximate momentum SGD by vanilla SGD with LR $\frac{\eta\lambda}{1-\gamma}$. Therefore, we propose the following discrete version of fast equilibrium conjecture.

For LR schedule $\eta(t)$ and WD schedule $\lambda(t)$, we define $\nu(\mu; \lambda,\eta, t)$ to be the marginal distribution of $(\vw_t, \vv_t)$ in the following dynamical system when $(\vw_0, \vv_0) \sim \mu$:
\begin{align*}
\vv_{t+1} &\gets \beta\vv_t + (\nabla \Loss(\vw_t; \Batch_t) + \lambda(t)\vw_t) \\
\vw_{t+1} &\gets \vw_t - \eta(t)\vv_{t+1}
\end{align*}

\begin{conjecture}[Fast Equilibrium Conjecture for Momentum] \label{conj:mom}
	For SGD with momentum, modern neural nets converge to the equilibrium distribution in $O(1/\lambda_e)$ time in the following sense. Given two initial distributions $\mu, \mu'$ for $\vw_0$, constant LR and effective WD schedules $\lambda^*, \eta^*$, there exists a mixing time $T = O(1 / \lambda^*_e)$, where $\lambda^*_e = \eta^*\lambda^*$, such that for any input data $\vx$ from some input domain $\mathcal{X}$,
	\[
		\dTV\left(P_{F\left(\vw_t;~\vx\right)}, P_{F\left(\vw'_t;~\vx\right)} \right) \approx 0,
	\]
	for all $t \ge T$, where $(\vw_t, \vv_t) \sim \nu(\mu; \lambda^*, \eta^*, t)$, $(\vw'_t, \vv'_t) \sim \nu(\mu'; \lambda^*, \eta^*, t)$.
	
	Moreover, let $\tilde{\eta}(\tau), \tilde{\lambda}(\tau)$ be a pair of LR and WD schedules, then there exists a mixing time $T = O(1 / \lambda^*_e)$ such that for any input data $\vx$ from some input domain $\mathcal{X}$,
	\[
		\dTV\left(P_{F\left(\vw_{t, \tau};~\vx\right)}, P_{F\left(\vw_{t, \tau}';~\vx\right)} \right) \approx 0
	\]
	for all $t \ge T$, where $\vw_{t,\tau} \sim \nu\left(\nu(\mu; \lambda^*, \eta^*, t); \tilde{\lambda}, \tilde{\eta}, t\right)$, $\vw'_{t, \tau} \sim \nu\left(\nu(\mu'; \lambda^*, \eta^*, t); \tilde{\lambda}, \tilde{\eta}, t\right)$.
\end{conjecture}

\subsection{Adam}

\begin{algorithm}  \label{alg:adam}
	\begin{algorithmic}[1]
		\STATE{\textbf{given} $\alpha = 0.001, \beta_1 = 0.9, \beta_2 =0.999, \epsilon = 10^{-8}, \lambda\in \R$} \label{adam-Given}
		\STATE{\textbf{initialize} time step $t \leftarrow 0$, parameter vector $\bm{\theta}_{t=0} \in \R^n$,  first moment vector $\vm_{t=0} \leftarrow \bm{0}$, second moment vector  $\vv_{t=0} \leftarrow \bm{0}$, schedule multiplier $\eta_{t=0} \in \R$}
		\REPEAT
		\STATE{$t \leftarrow t + 1$}
		\STATE{$\nabla f_t(\bm{\theta}_{t-1}) \leftarrow  \text{SelectBatch}(\bm{\theta}_{t-1})$}  \COMMENT{select batch and return the corresponding gradient}
		\STATE{$\vg_t \leftarrow \nabla f_t(\bm{\theta}_{t-1})$  \adam{+ \lambda\bm{\theta}_{t-1}}}
		\STATE{$\vm_t \leftarrow \beta_1 \vm_{t-1} + (1 - \beta_1) \vg_t $} \label{adam-mom1} \COMMENT{here and below all operations are element-wise}
		\STATE{$\vv_t \leftarrow \beta_2 \vv_{t-1} + (1 - \beta_2) \vg^2_t $} \label{adam-mom2}
		\STATE{$\hat{\vm}_t \leftarrow \vm_t/(1 - \beta_1^t) $} \COMMENT{$\beta_1$ is taken to the power of $t$} \label{adam-corr1}
		\STATE{$\hat{\vv_t} \leftarrow \vv_t/(1 - \beta_2^t) $} \COMMENT{$\beta_2$ is taken to the power of $t$} \label{adam-corr2}
		\STATE{$\eta_t \leftarrow \text{SetScheduleMultiplier}(t)$}\COMMENT{can be fixed, decay, or also be used for warm restarts}
		\STATE{$\bm{\theta}_t \leftarrow \bm{\theta}_{t-1} - \eta_t \left(\alpha\hat{\vm}_t / (\sqrt{\hat{\vv}_t} + \epsilon) \ouradam{+ \lambda\bm{\theta}_{t-1}} \right)$} \label{adam-xupdate}
		\UNTIL{ \textit{stopping criterion is met} }
		\RETURN{optimized parameters $\bm{\theta}_t$}
	\end{algorithmic}
	\caption{\scriptsize \adamtext{Adam with L$_2$ regularization} and \ouradamtext{Adam with decoupled weight decay (AdamW)} [Copied from \citep{loshchilov2018decoupled} ]}
\end{algorithm}

\myparagraph{Connection to AdamW:} \citep{loshchilov2018decoupled} found that using the parametrization of $\lambda_e = \eta\lambda$  achieves better generalization and a more separable hyper-parameter search space for SGD and Adam, which are named SGDW and AdamW respectively. So far we have justified the role of intrinsic LR for SGD(W). The theorem below shows that the notion of the intrinsic LR also holds for AdamW, while the learning rate has no more power than initialization scale.

\begin{theorem}
	For fixed scale-invariant losses $\{f_t(\vw)\}_{t=1}$, constant schedule multiplier $\eta_t\equiv \eta$ and $\eps=0$, multiplying the initial weight $\vW_0$ and the learning rate $\alpha$ by the same constant $C$ would not change the trajectory of AdamW (\Cref{alg:adam}) in function space.
\end{theorem}

\begin{remark}
The notation of LR is slightly different in \citep{loshchilov2018decoupled} than in the main paper, where $alpha$ is LR and $\eta_t$ is the \textsc{Schedule Multiplier}. By using schedule multiplier, AdamW\Cref{alg:adam} can decay LR and WD factor simultaneously. This notation is only used for the statement of the above theorem and its proof. 	
\end{remark}

\begin{proof}
	The proof is based on induction. It suffices to prove that for two history $\{\vtheta_t\}_{t=0}$ and $\{\vtheta'_t\}_{t=0}$ satisfying that $C\vtheta_t = \vtheta_t'$, for all $0\le t\le T$ and  some $C>0$, and evolving with $\alpha$ and $\alpha' = C\alpha$ respectively, the following holds:
	\[ C\vtheta_{T+1}=\vtheta'_{T+1}.\]
	
Note that by scale invariance of $f_t$, $\vg_t = C\vg'_t, \forall 1\le t\le T$, therefore by definition $\alpha\hat{\vm}_T / \sqrt{\hat{\vv}_T} = \alpha\hat{\vm'}_T / \sqrt{\hat{\vv}'_T}$, i.e. it's independent of scaling of the history. We now can conclude that 
\[ C\vtheta_{T+1} = C\vtheta_{T} -\eta C\alpha \hat{\vm}_T / \sqrt{\hat{\vv}_T} -\eta C\vtheta_{T} = \vtheta'_{T} -\eta \alpha' \hat{\vm}'_T / \sqrt{\hat{\vv}'_T} -\eta \vtheta'_{T}= \vtheta_{T+1},\]
which completes the proof.
\end{proof}

\section{Discussion on the Benefit of Early Large Intrinsic LR}\label{appsec:benefit_early_large_lr}

Fast equilibrium conjecture says that the equilibrium can be reached in $O(1/\lambda_e)$ steps for all reasonable initializations. Indeed, \Cref{eqn:norm_mixing2} indicates that there is also a logarithmic dependency on $\frac{\gamma_t}{\gamma_0}$, i.e., if the initial effective LR is far from the effective LR at equilibrium, then the mixing time can be larger by a multiplicative constant compared to good initial effective LRs. Below we show this constant improvement coming from a good initialization matters a lot for real life training (meaning the training budget is limited), and the usage of initial large intrinsic LR helps SGD to reach a better initialization for the final phase, and thus allow faster mixing to the final equilibrium.

\paragraph{The benefit of early large intrinsic learning rates.} In this section we give experimental evidence that how the fast equilibrium conjecture led by BatchNorm    + WD makes the Step Decay training schedule robust to various different initialization methods. In detail, we compare the following 4 types of initialization: Neural Tangent Kernel (NTK) initialization~\citep{jacot2018ntk,arora2019exact}, Kaiming initialization~\citep{he2015delving} , Kaiming initialization multiplied by 1000 and  Kaiming initialization multiplied by 0.001. In \Cref{fig:unfavorable_init} we show that the initial large (intrinsic) learning rate in Step Decay is very necessary to ensure SGD reach the equilibrium of small (intrinsic) LR within the normal training budget, and thus achieving good test accuracy.

\paragraph{Comparison between the above four initialization.} Briefly speaking, these methods are quite similar as they all initialize each parameter by i.i.d. Gaussian, and the only difference is the variance of the gaussian distribution in each layer. For Kaiming initialization, the variance is roughly $\frac{1}{N}$, and for NTK initialization, the variance is always $O(1)$ but there is an additional multiplier of $O(\frac{1}{\sqrt{N}})$ per layer, where $N$ is the number of the input channels/neurons that layer. \footnote{Strictly speaking, NTK initialization is a re-parametrization of Kaiming initialization,  rather than a different initialization method, as the additional multiplier indeed changes the architecture.}
Note that the NTK initialization and Kaiming initialization are always the same in function space. Due the scale invariance led by BatchNorm, all the scaled version of Kaiming initialization are the same as the original Kaiming initialization in function space. 

\begin{figure}[!htbp]
	\centering
	\begin{subfigure}[b]{1.0\textwidth}
		\centering
		\includegraphics[width=\linewidth]{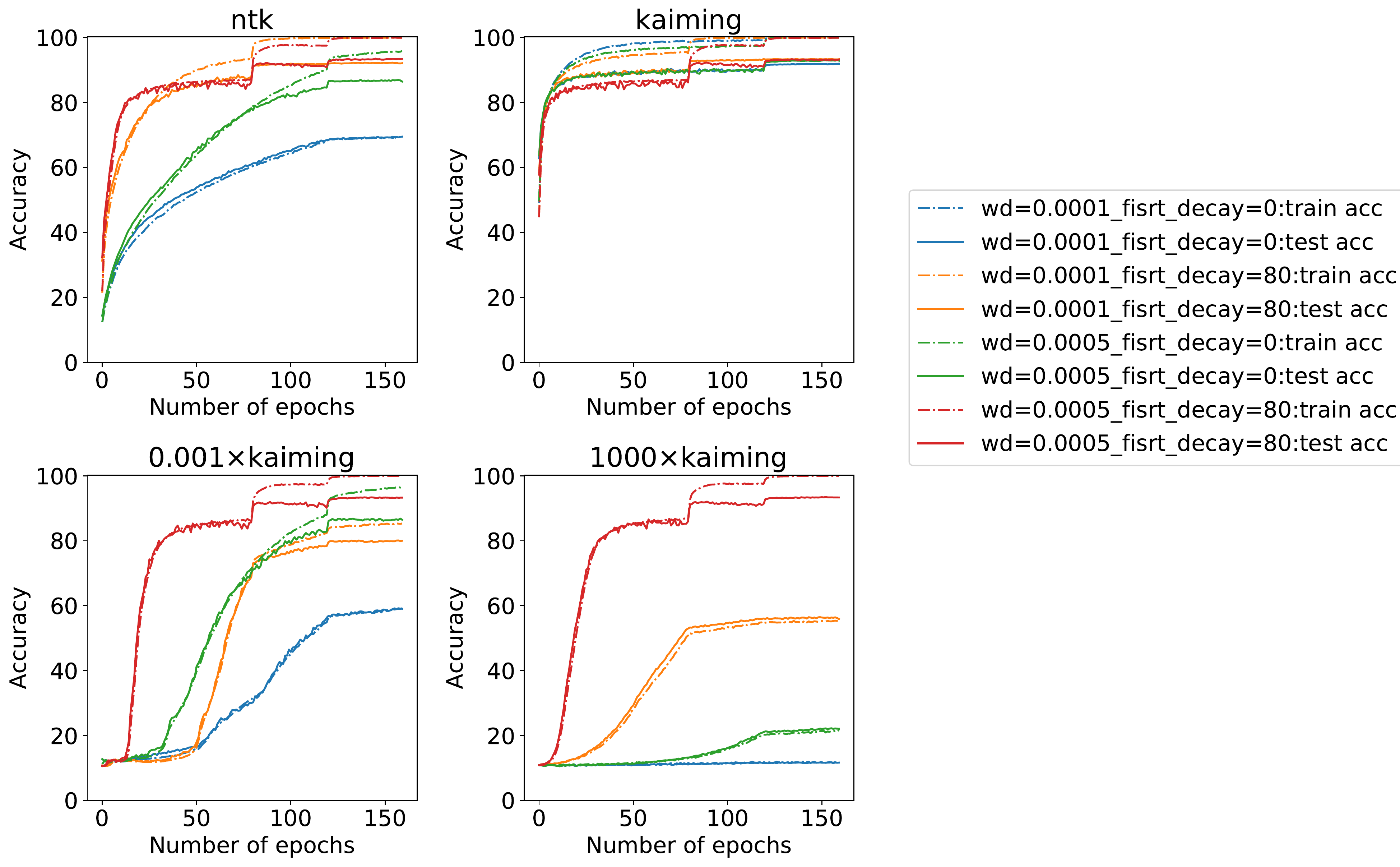}
		\caption{Train/test Accuracy}
	\end{subfigure}%
	
	\begin{subfigure}[b]{1.0\textwidth}
		\centering
		\includegraphics[width=\linewidth]{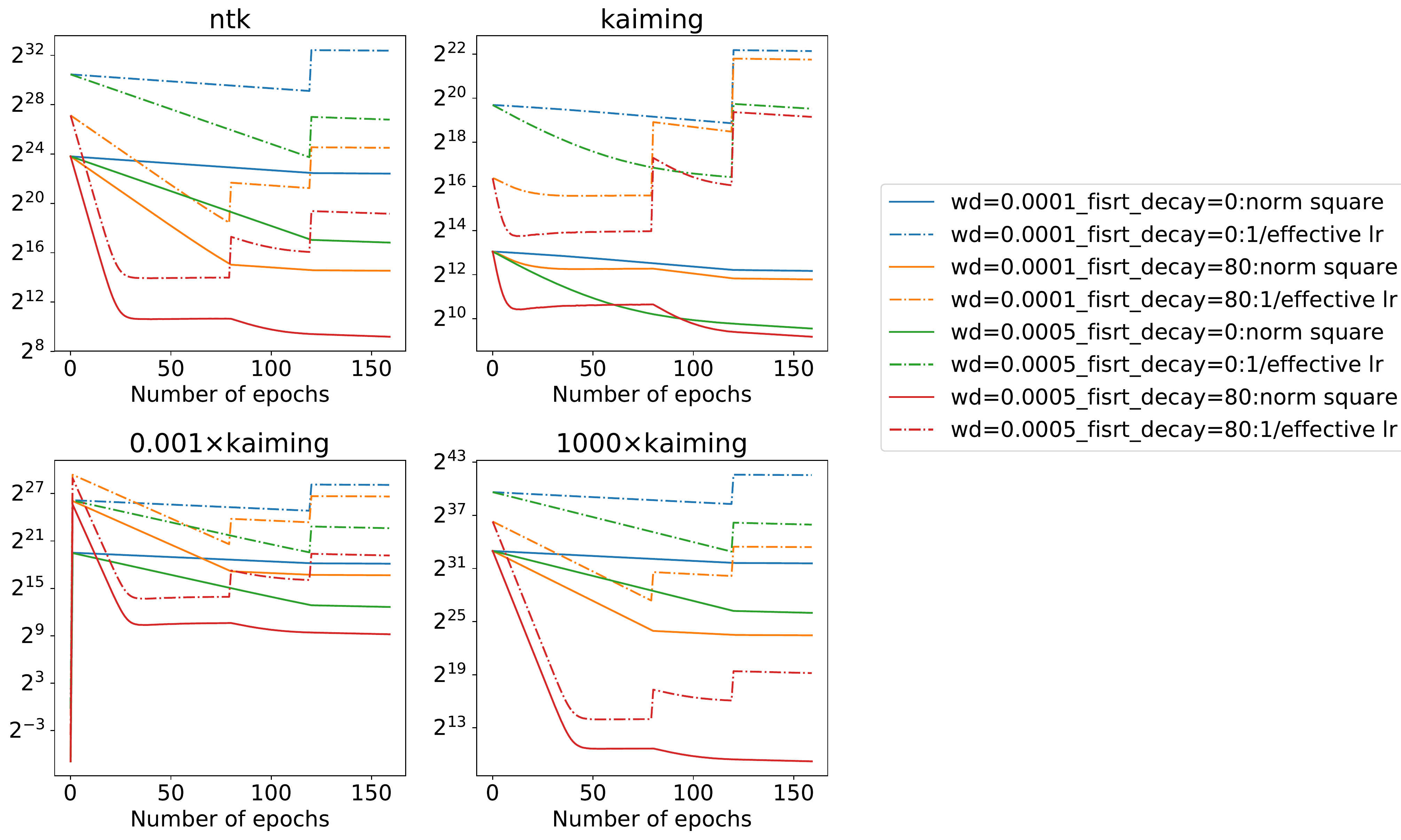}
		\caption{ Norm and effective LR}
	\end{subfigure}%
	\caption{The large initial intrinsic LR (as well as large WD factor) helps achieve high test accuracy within normal training budget consistently for different initialization methods. The training curve and convergence time to equilibrium for large initial LR is robust even to the extreme small/large initializations. PreResNet32 trained by momentum SGD with initial LR $0.1$ on CIAFR10 with 4 different initialization methods, 2 different WD values, and 2 different LR schedules. Each LR schedule divides its LR by 10 twice at epoch [80,120] (the normal schedule) or epoch [0,120] (meaning starting with a 10 times smaller LR, $0.01$). The red line and orange line performs much better than their counterparts (without initial large LR) when not using standard Kaiming Initialization. Still, the red line even outperforms orange line a lot when the initialization are extremely large or small, due to the effect of large intrinsic LR brought by large WD factor. This justifies the argument in \Cref{sec:theory} that the equilibrium of small intrinsic LR is much closer to that of large intrinsic LR, than some arbitrary random initialization. This is very clear from the view of norm convergence. See (b). }
	\label{fig:unfavorable_init}
\end{figure}

\paragraph{A Theoretical Analysis on Norm Convergence.} Although the convergence of norm is not equivalent to the convergence in function space, analysing the convergence of norm can provide insights into how large LR helps training. Now we theoretically analyse the effect of early large LR on the convergence rate of norm. We compare the following two processes with the same initial norm squared $G_0$:
\begin{enumerate}
	\item Train the neural net with LR $\eta$;
	\item Train the neural net with intrinsic LR $K\eta$, then decay it to $\eta$ after the norm converges.
\end{enumerate}
For simplicity, we consider the case that $\frac{\sigma^2}{2 \eta \lambda} = 1$, which means $\gamma_t$ in the first process eventually converges to $1 + O(\epsilon)$; other cases can be transformed to this case by re-scaling the initialization.

For the first process, $\gamma_t = G_0 / \eta^2$ initially. By~\Cref{lm:gamma-decay}, $\gamma_t$ converges to $1 + O(\epsilon)$ in
\[
	O\left(\frac{1}{\eta \lambda} \max\left\{\ln \frac{G_0}{\eta^2}, 1\right\} \right)
\]
time. For the second process, $\gamma_t = G_0 / (K^2 \eta^2)$ initially, and $\gamma_t$ first converges to $(1 + O(\epsilon)) \frac{1}{K}$ in $O\left(\frac{1}{K\eta \lambda} \max\left\{\ln \frac{G_0}{K\eta^2}, 1\right\} \right)$ time. After LR decay, $\gamma_t$ instantly becomes $(1 + O(\epsilon)) K$ as $\gamma_t$ is inversely proportional to LR squared. Then we only need another $O(\frac{1}{\eta \lambda} \ln K)$ time to make the effective LR converges again. Overall, the second process takes
\begin{equation}\label{eq:how_norm_change_two_phase}
	O\left(\frac{1}{K\eta \lambda} \max\left\{\ln \frac{G_0}{K\eta^2}, 1\right\} + \frac{1}{\eta \lambda} \ln K\right)
\end{equation}
time. Comparing the second process with the first process, we can see that the large initial LR reduces the dependence of convergence time on the initial norm. It is worth to note that $\frac{G_0}{\eta^2}$ is typically larger than $K$ (which equals to $10$) in~\Cref{fig:unfavorable_init}. Therefore, a large initial LR also leads to faster convergence time without tuning the initialization scale.

\myparagraph{Explanation for different convergence rates in \Cref{fig:unfavorable_init}:} The 4 settings about LR schedules and WD can be interpreted using \Cref{eq:how_norm_change_two_phase} as the choices of $(K,\lambda)$. Let $\eta=0.01$, $K=1$ means starting with $\eta=0.01$, while $K=10$ means starting with the default LR, $0.1$, which is $10$ times larger than $\eta$. For the rest 3 initializations other than kaiming initialization, from \Cref{fig:unfavorable_init}, we can see that the initial norm are all exponentially large\footnote{As discussed earlier, NTK initialization has larger weight norm. For 0.001 kaiming initialization, the reason is more subtle: the initial norm are indeed super small, thus leading to huge initial gradient, and therefore the norm grows quickly in the first few iterations.}, making $\ln \frac{G_0}{K\eta^2}$ a large constant. Thus the total steps of training has to be $\Omega(\frac{1}{K\eta \lambda} \ln \frac{G_0}{K\eta^2})$ for the effective learning rate to grow and the training to proceed. This could also be seen directly from the ratio of the slopes of the log norm square, which is ${\color{blue}1}:{\color{orange}10}:{\color{green}5}:{\color{red}50}$.
\section{Supplementary Figures and Tables for \Cref{sec:evidence}}

\subsection{Equilibrium is Independent of Initialization}

This subsection provides supplementary materials to justify that the equilibrium is independent of initialization.

\Cref{table:random_history} shows the LR and WD of each random schedule in \Cref{fig:random_hist}. \Cref{fig:random_hist_vgg16_cifar10} and \Cref{fig:random_hist_cifar100} are experiments in similar settings as \Cref{fig:random_hist} to show that the equilibrium is independent of initialization for VGG16 on CIFAR-100.

We also validate our claim in the case that we initialize the training with a single possible initial point $\vw_0$ in a similar setting as \Cref{fig:mnist}. That is, we first randomly sample a parameter from the distribution for random initialization, and use it to initialize CNNs in all the independent runs for estimating the equilibrium. \Cref{fig:mnist-tv-fixed} shows that CNNs still converge to the equilibrium even if the initial parameter $\vw_0$ is fixed to the same random sample.

\begin{table}[!htbp]
	{\scriptsize
		\centering
		\begin{tabular}{c||c |c|c|c|c|c|c}
			Epoch & 0 &  100 & 200 & 300 & 400 & 500 & 1000 \\
			\hline
			\blue{Schedule\_1}  &  - & LR$/4$ & LR$\times 4$  & LR$/4$ & LR$\times2$ & LR$\times2$ & LR$/10$,WD $=0$\\
			\orange{Schedule\_2} & - & - & -  & -& - & - & LR$/10$,WD $=0$\\
			\green{Schedule\_3} & - & LR$\times4$ & LR$/2$ &  LR$/2$ &  LR$/4$ &  LR$\times 4$ & LR$/10$,WD $=0$\\
			\red{Schedule\_4} & -  & LR,WD$\times4$ & LR,WD$/2$  & LR,WD$/2$ & LR,WD$/4$ & LR,WD$\times 4$ &LR$/10$,WD $=0$\\
			\purple{Schedule\_5} & LR$\times32$  & LR$/2$ & LR$/2$  &  LR$/2$ & LR$/2$  & LR$/2$ &LR$/10$,WD $=0$\\
			\hline
		\end{tabular}
		\vspace{0.1cm}
		\caption{LR/WD Schedules in \Cref{fig:random_hist}. All the schedules have the same initial LR $= 0.4$ and classic WD $=0.0005$. The batch size is 128 and momentum is turned off.}
		\label{table:random_history}
	}
\end{table}

\begin{figure}[htbp] \label{fig:mnist-fixed}
	\centering
	\includegraphics[scale=0.4]{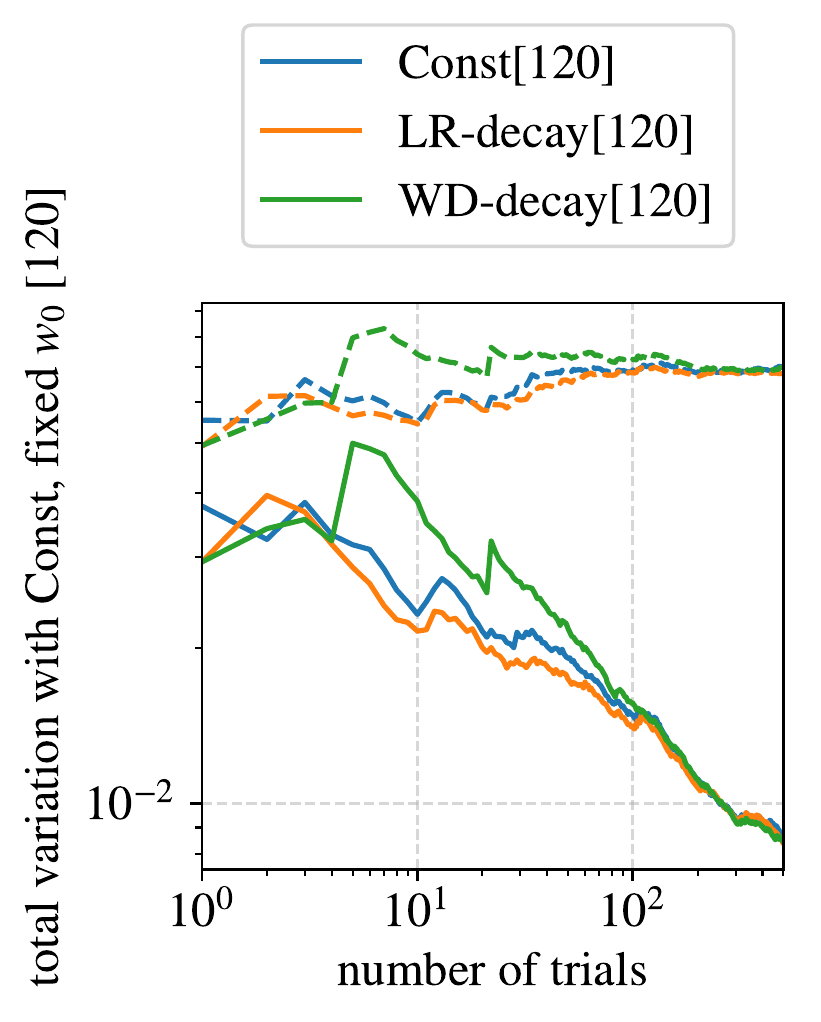}
	\includegraphics[scale=0.4]{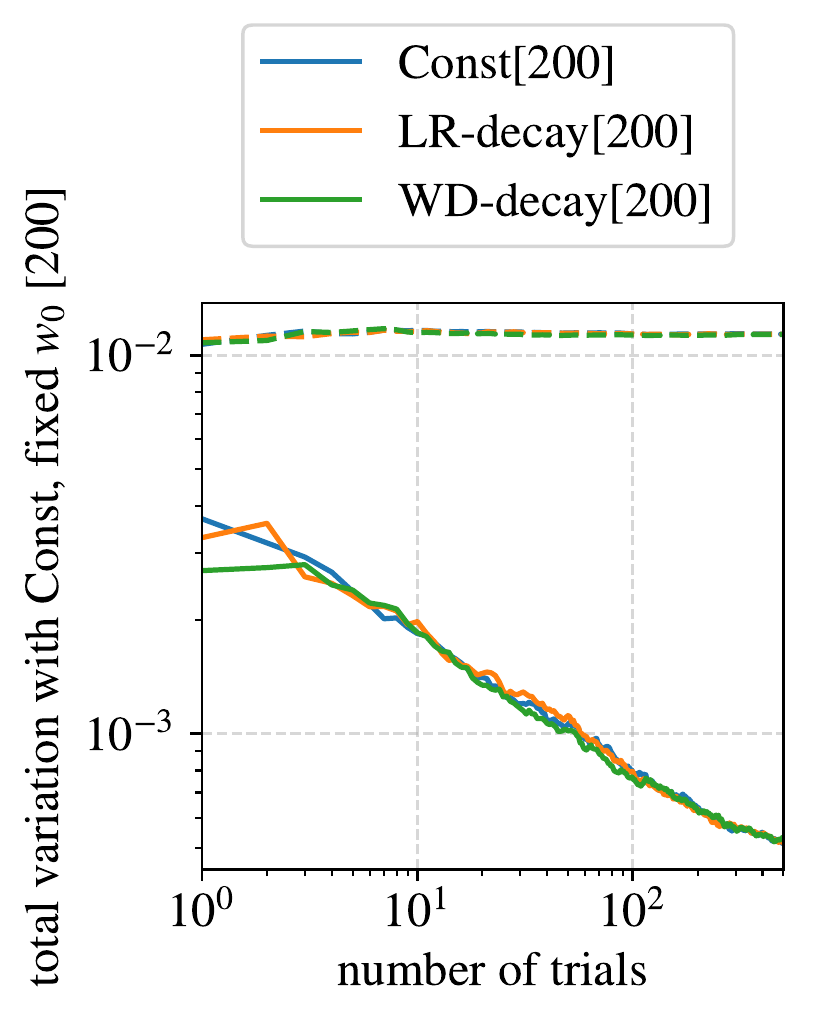}
	\caption{\footnotesize CNNs trained on MNIST converge to the equilibrium even if the initial parameter $\vw_0$ is fixed to some random sample. We estimate the total variation between the empirical distribution of the predictions on test images for neural nets trained with schedule \texttt{Const} with fixed $\vw_0$ and other schedules for 120/200 epochs (solid lines). The estimated value decreases with the number of trials. The dashed lines are the sum of averaged test errors of each pair of training processes which can be seen as baselines.} \label{fig:mnist-tv-fixed}
\end{figure}

\subsection{Equilibrium Can be Reached in $O(1/\lambda\eta)$ Steps}

In this subsection we provide more experimental evidence that the mixing time to equilibrium in function space scales to $\frac{1}{\eta\lambda}$. Note in \eqref{eqn:norm_mixing2}, the convergence of norm also depends on the initial value. Thus in order to reduce the effect of initialization on the time towards equilibrium, we use the setting of Figure 3 in \citep{li19learningrate}, where we first let the networks with the same architecture reach the equilibrium of different intrinsic LRs, and we decay the LR by $10$ and multiplying the WD factor by $10$ simultaneously. In this way the intrinsic LR is not changed and the equilibrium is still the same. However, the effective LR is perturbed far away from the equilibrium, i.e. multiplied by $0.1$. And we measure how long does it takes SGD to recover the network back to the equilibrium and we find it to be almost linear in $1/\lambda\eta$.

 \begin{figure}[!htbp]
 	\centering
 	\begin{subfigure}[b]{0.33\textwidth}
 		\centering
 		\includegraphics[width=\linewidth]{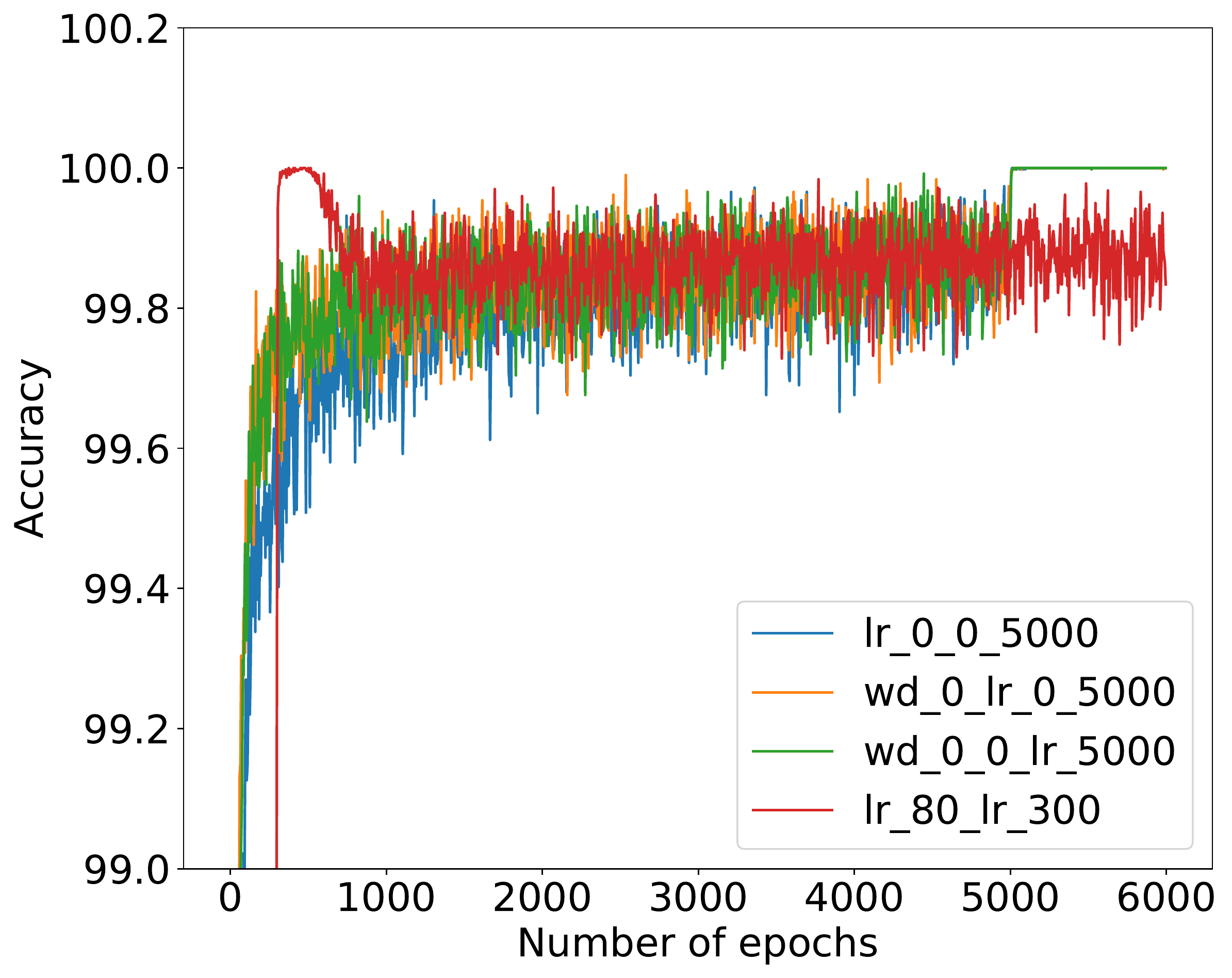}
 		\caption{Train accuracy}
 	\end{subfigure}%
 	\begin{subfigure}[b]{0.33\textwidth}
 		\centering
 		\includegraphics[width=\linewidth]{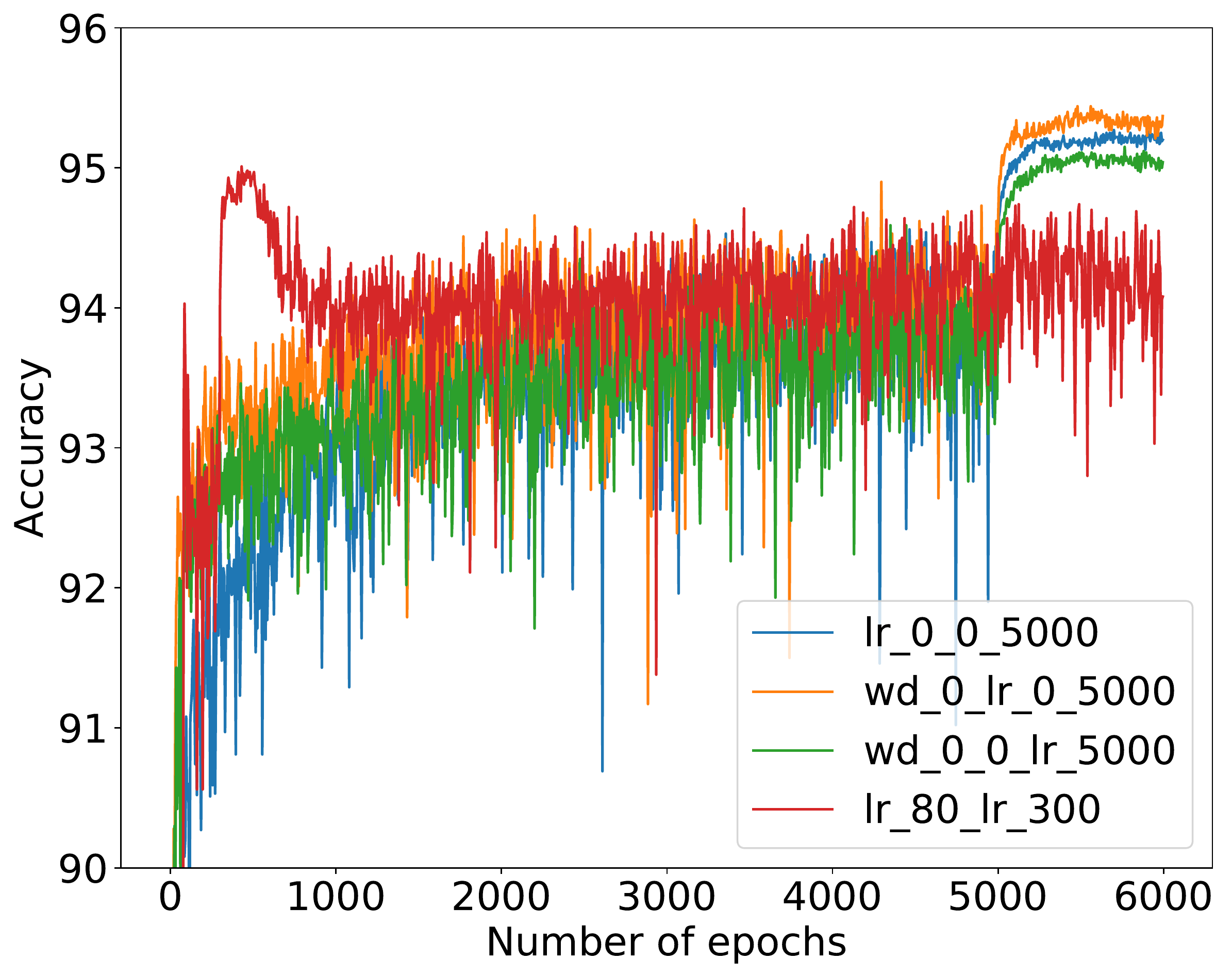}
 		\caption{Test accuracy}
 	\end{subfigure}%
 	\begin{subfigure}[b]{0.33\textwidth}
 		\centering
 		\includegraphics[width=\linewidth]{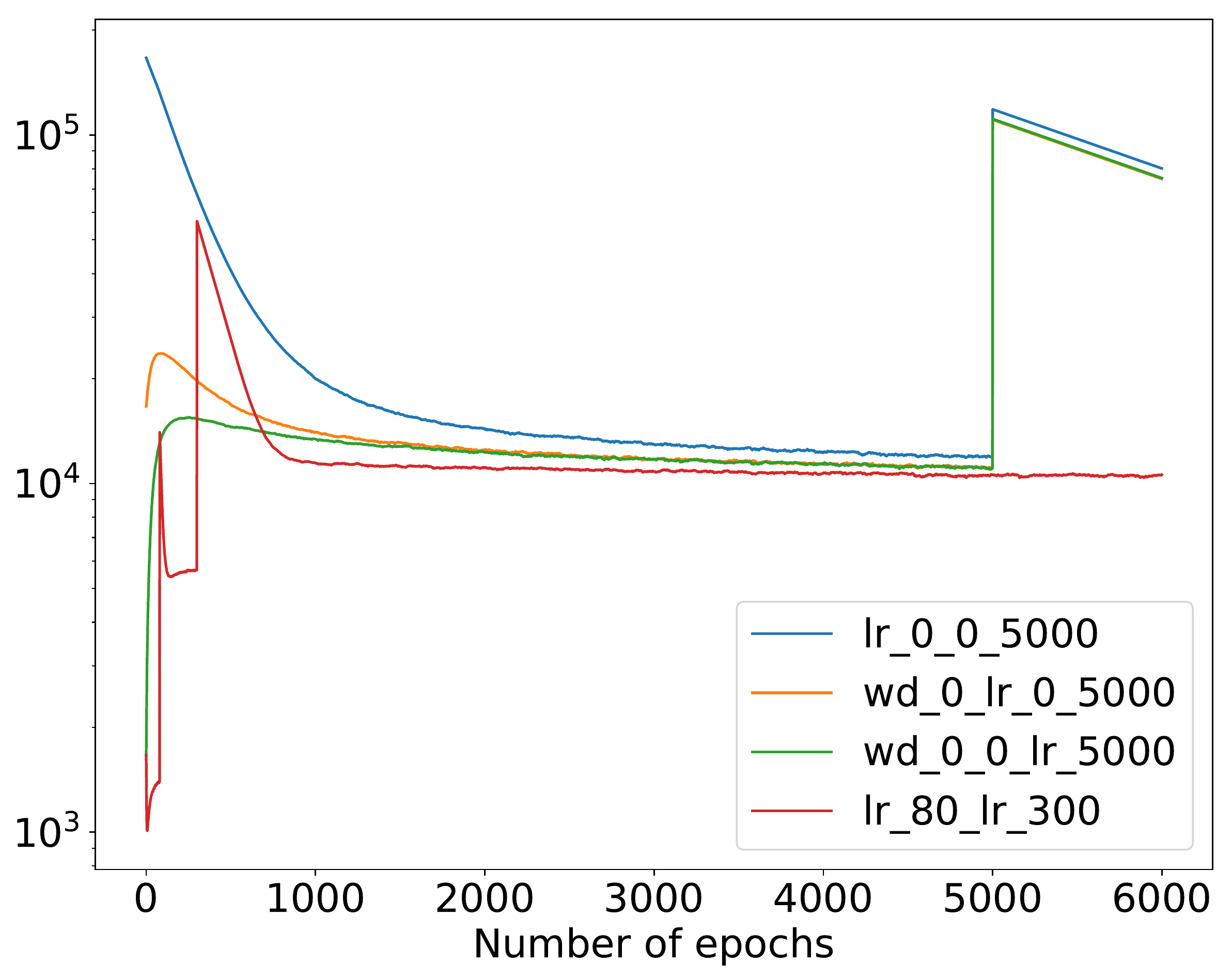}
 		\caption{Effective LR, $\gamma_t^{-1/2}$}
 		\label{fig:small_lr_sota_norm}
 	\end{subfigure}%
 	\caption{\small  Achieving SOTA test accuracy by $0.9$-momentum SGD with small learning rates (the blue line). The initial learning rate is 0.1, initial WD factor is 0.0005. The label \texttt{wd\_x\_y\_lr\_z\_u} means dividing WD factor by 10 at epoch $x$ and $y$, and dividing LR by 10 at epoch $z$ and $u$. For example, the blue line means dividing LR by 10 twice at epoch 0, i.e. using an initial LR of 0.01 and dividing LR by 10 at epoch 5000. The red line is baseline. }
 	\label{fig:small_lr_sgd_sota}
 \end{figure}
 
\begin{figure}[!htbp]
	\centering
	\begin{subfigure}[b]{0.75\textwidth}
		\centering
		\includegraphics[width=\linewidth]{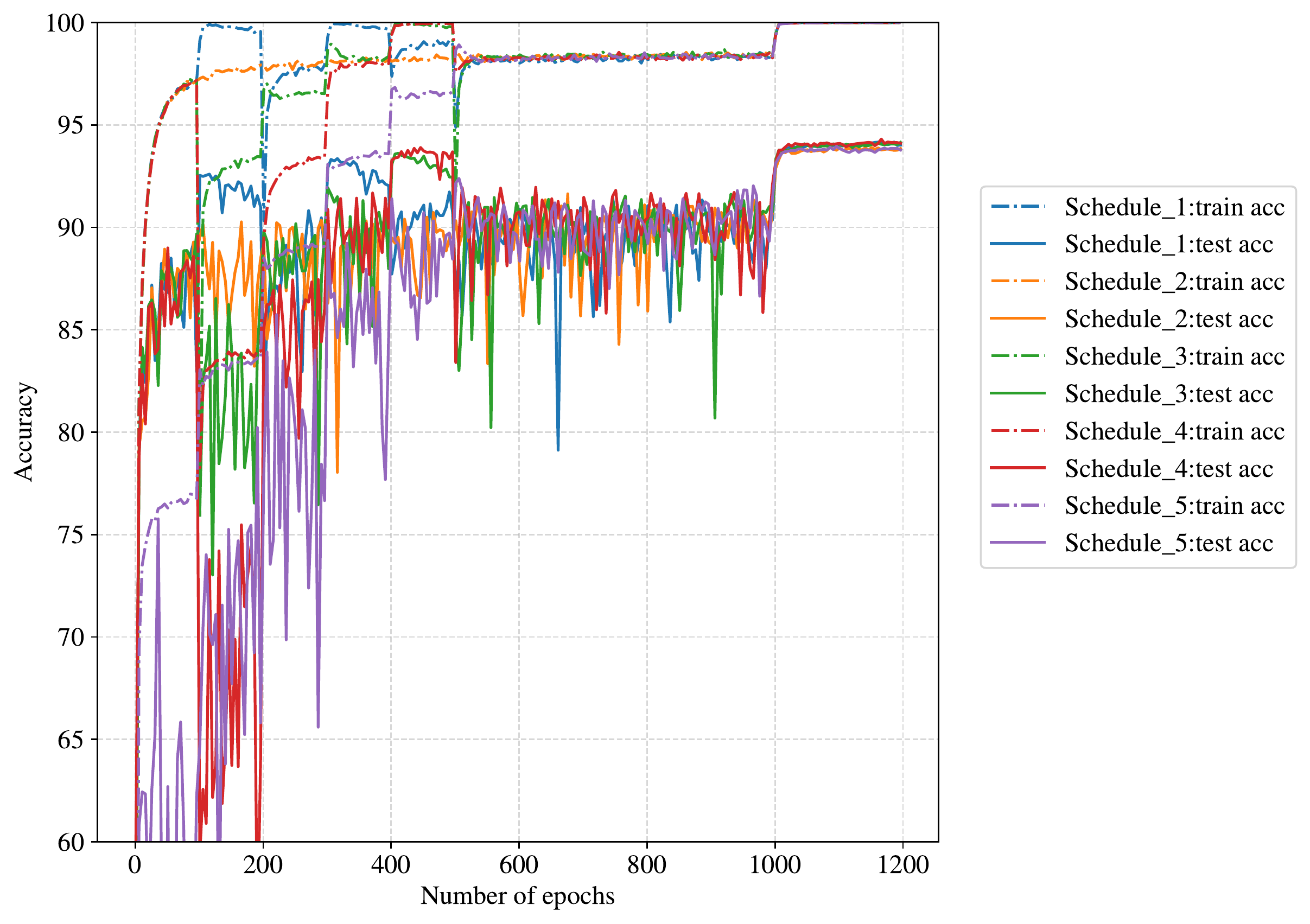}
		\caption{Train/test accuracy}
	\end{subfigure}%
	
	\begin{subfigure}[b]{0.75\textwidth}
		\centering
		\includegraphics[width=\linewidth]{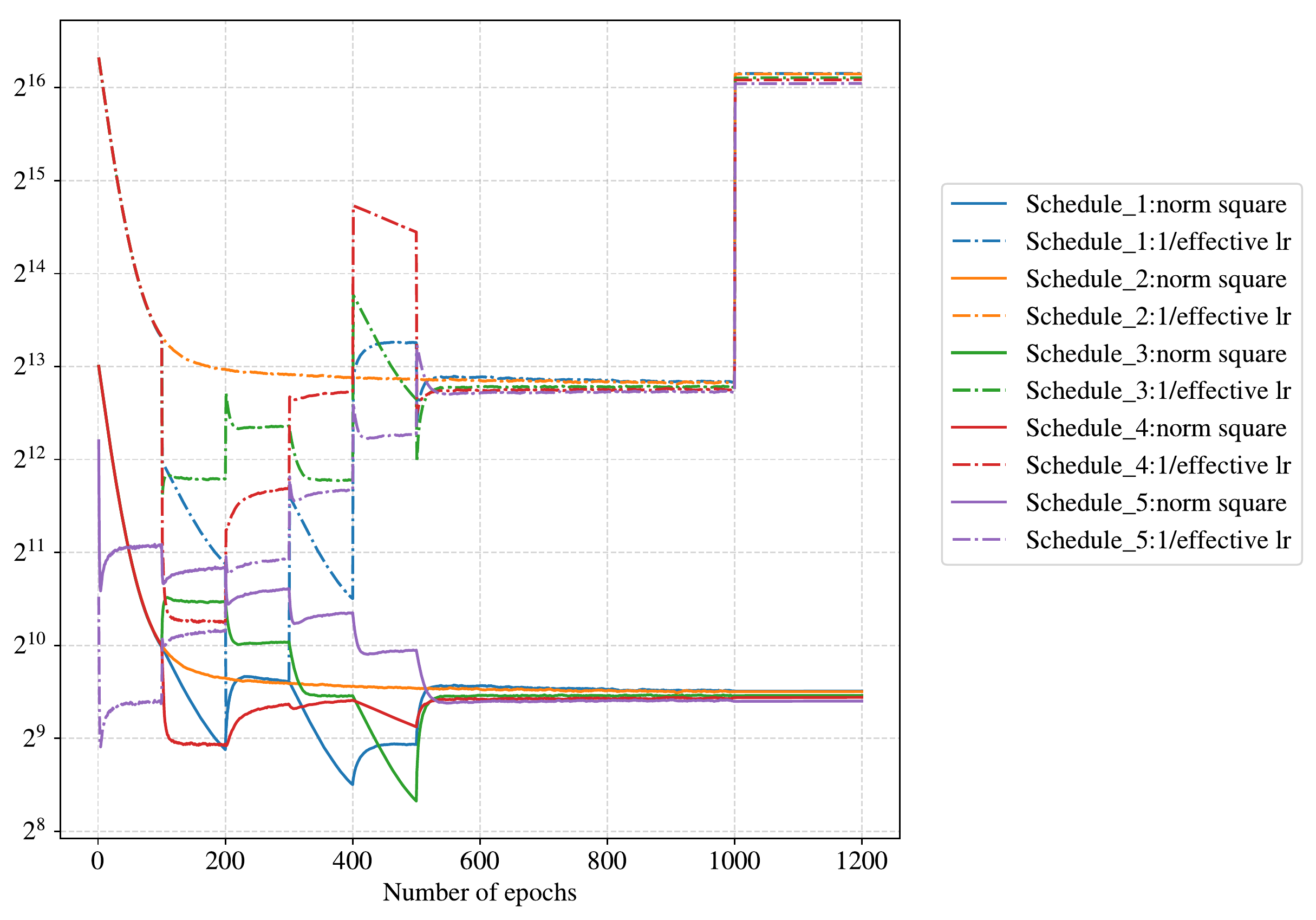}
		\caption{Norm and effective LR}
	\end{subfigure}%
	\caption{VGG16 trained by SGD on CIAFR10 with 5 random LR/WD schedules in \Cref{table:random_history}, same as that in \Cref{fig:random_hist}. These different trajectories exhibit similar test/train accuracy,  norm and effective LR after switching to the same intrinsic LR at epoch 500. Moreover, they achieve the same best test accuracy ($\sim 94\%$) after decaying LR and removing WD at epoch 1000. This again supports the conjecture that the equilibrium is independent of initialization.}
	\label{fig:random_hist_vgg16_cifar10}
\end{figure}

\begin{figure}[!htbp]
	\centering
	\begin{subfigure}[b]{0.8\textwidth}
		\centering
		\includegraphics[width=\linewidth]{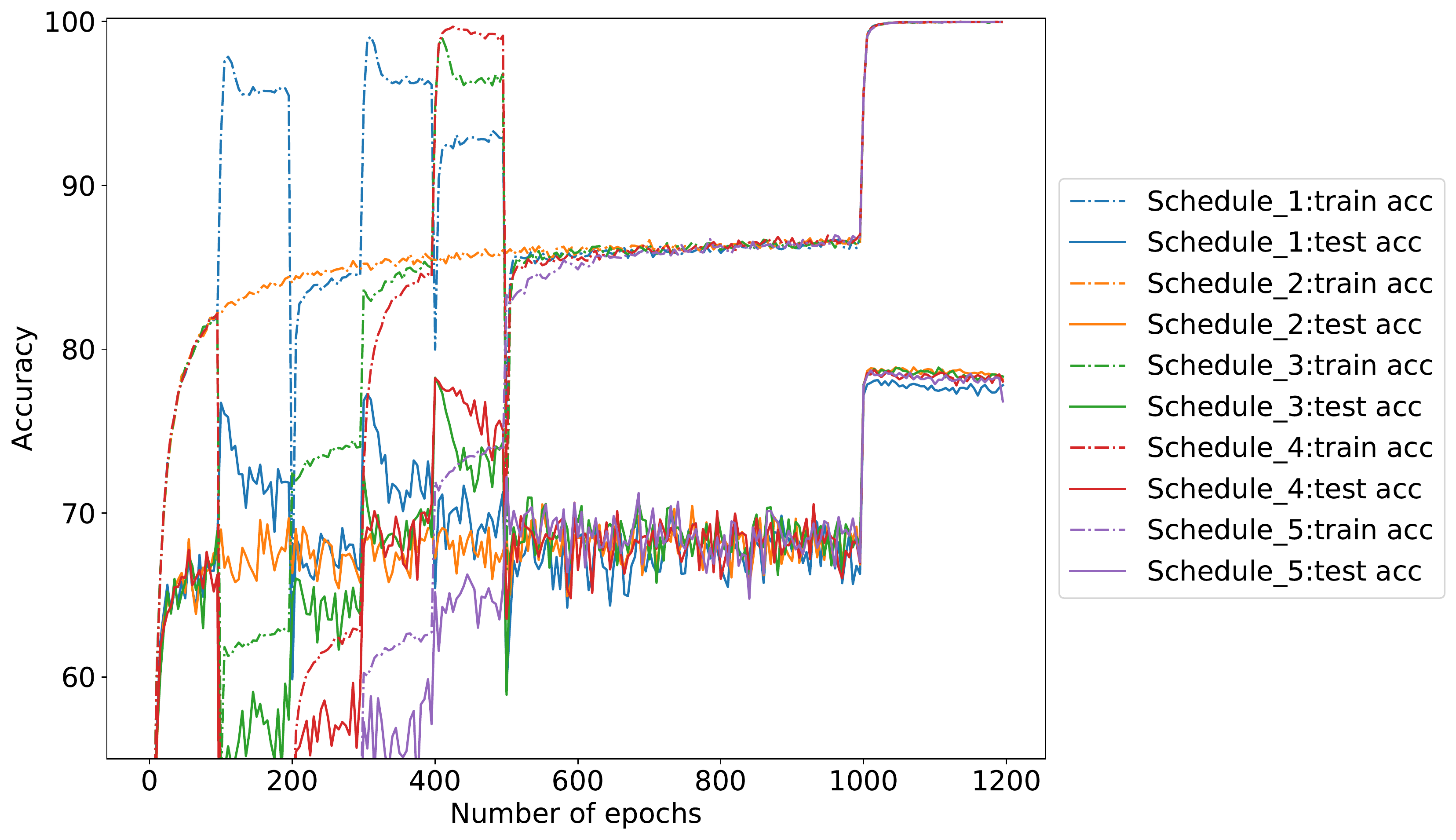}
		\caption{Train/test accuracy}
	\end{subfigure}%
	
	\begin{subfigure}[b]{0.8\textwidth}
		\centering
		\includegraphics[width=\linewidth]{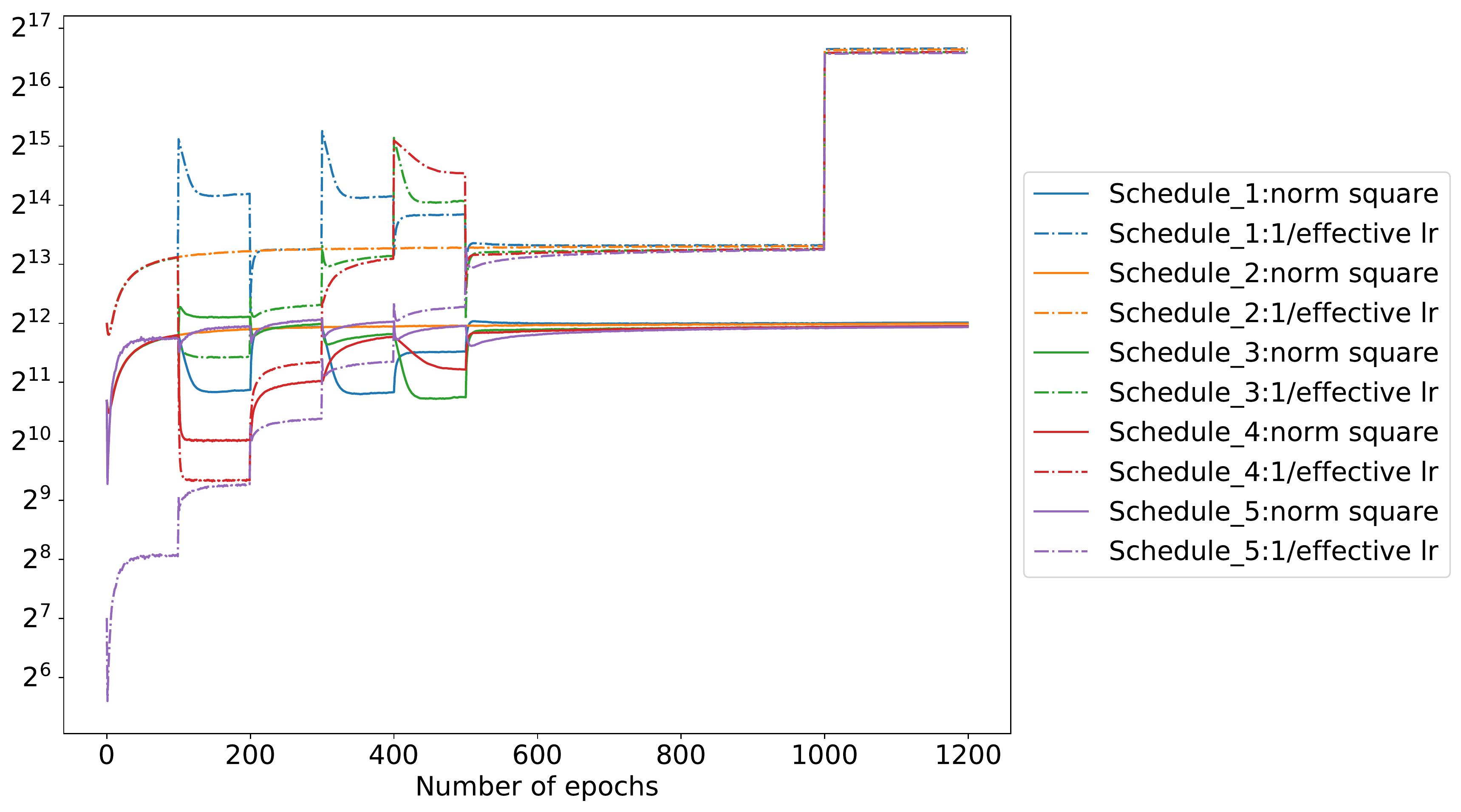}
		\caption{Norm and effective LR}
	\end{subfigure}%
	\caption{PreResNet32 trained by SGD on CIAFR100 with 5 random LR/WD schedules in \Cref{table:random_history}, same as that in \Cref{fig:random_hist}. These different trajectories exhibit similar test/train accuracy,  norm and effective LR after switching to the same intrinsic LR at epoch 500. Moreover, they achieve the same best test accuracy ($\sim 78\%$) after decaying LR and removing WD at epoch 1000, thus supporting the conjecture that the equilibrium is independent of initialization.}
	\label{fig:random_hist_cifar100}
\end{figure}

\end{document}